\documentclass[12pt,twoside]{article}
  \newdimen\paravsp  \paravsp=1.3ex 
\usepackage{latexsym,amsmath,amssymb}
\usepackage{hyperref}
\usepackage{wrapfig} 
\usepackage{graphicx} 
\usepackage[all]{xy}  
\usepackage{tikz}     
\usepackage{pgfplots} 

\usepackage{amsthm}   
\usepackage{multirow}

\numberwithin{equation}{section}
\theoremstyle{definition} 
\newtheorem{theorem}{Theorem}[section]
\newtheorem{lemma}{Lemma}
\newtheorem{corollary}[theorem]{Corollary}
\newtheorem{orexample}{Example}
\newtheorem{proposition}{Proposition}[section]
\newenvironment{definition}[1][Definition.]{\begin{trivlist}
\item[\hskip \labelsep {\bfseries #1}]}{\end{trivlist}}

\newenvironment{keywords}{\centerline{\bf\small
Keywords}\begin{quote}\small}{\par\end{quote}\vskip 1ex}

\newenvironment{example}[1][]
{\begin{orexample}[#1]}
{\hspace*{\fill}\ensuremath{\diamondsuit\quad}
\end{orexample}}

\newtheorem{myexample}[theorem]{Example}

\def \bigss{{\Omega}_{\alpha:\beta}} 
\def \permute{{\pi}} 
\def \measure{{pr}} 
\def \bigmeasure{{ \text{Pr} }} 
\def \Qtt    {F\!G}
\def \Qtf    {F\!\bar{G}}
\def \Qft    {\bar{F}\!G}
\def \Qff    {\bar{F}\!\bar{G}}
\def \impl   {F \!{\scriptstyle \rightarrow} G}
\def \inds   {{\cal{I}}}	
\def \preds {{\cal{A}}}	
\def \usize{N}

\begin{document}

\title{\vspace{-4ex}
\vskip 2mm\bf\Large\hrule height5pt \vskip 4mm
On Nicod's Condition, Rules of Induction and the Raven Paradox
\vskip 4mm \hrule height2pt}
\author{{\bf Hadi Mohasel Afshar \& Peter Sunehag}\\[3mm]
\normalsize Research School of Computer Science \\[-0.5ex] 
\normalsize Australian National University \\[-0.5ex]
\normalsize Canberra, ACT, 0200, Australia \\
\normalsize \texttt{hadi.afshar@anu.edu.au \qquad peter.sunehag@anu.edu.au}
}
\date{}
\maketitle

\begin{abstract}
Philosophers writing about the ravens paradox often note that 
Nicod's Condition (NC) holds given some set of background 
information, and fails to hold against others, but rarely go any further. 
That is,  it is usually not explored which 
background information makes NC true or false.   
The present paper aims to fill this gap. 
For us, ``(objective) background knowledge'' is restricted to  
information that can be expressed as probability events. 
Any other configuration is 
regarded as being subjective and a property of the a priori probability distribution.
We study NC in two specific 
settings. 
In the first case, a \emph{complete description} of some individuals is known, 
e.g.\ one knows of each of a group of individuals whether they 
are black and whether they are ravens. In the second case, the number of
individuals having a particular property is given, e.g.\ one knows 
how many ravens or how many black things there are (in the relevant population).
While some of the most famous answers to the paradox are measure-dependent, 
our discussion is not restricted to any particular probability measure. 
Our most interesting result is that in the second setting, NC violates a simple kind of inductive inference (namely \emph{projectability}). 
Since relative to NC, this latter rule is more closely related to, and more directly justified by our intuitive notion of inductive reasoning, this tension makes a 
case against the plausibility of NC. 
In the end, we suggest that the
informal representation of NC may seem to be intuitively plausible
because it can easily be mistaken for \em{reasoning by analogy}.
\end{abstract}


\newpage

\section{Introduction} \label{sect:introduction}

In this article, we study induction and in particular Nicod's
Condition (NC) from a Bayesian (in the sense of \emph{subjective probability}) point of view. 
Rules of induction can be thought
of as such restrictions on the class of probability measures (or equivalently, on the class of \emph{rational agents}\footnote{
\label{foot:peter}
The term \emph{rational agent} (that is, the performer of induction), as used in fields such as \emph{decision theory} and \emph{artificial intelligence} \cite{russell03}, by definition, refers to an agent that satisfies certain consistency axioms (see \cite{savage54}). 
Representation theorems show that this implies that the agent has a probability distribution representing a priori beliefs for what will happen and utilities for the various possible outcomes, such that decisions can be explained as maximizing expected utility. Given utilities we can from simpler axioms \cite{sunehag11} infer the existence of a probability distribution.
}
).

The question is: ``How can we agree that any particular rule of
induction is plausible and generally entails sensible consequences
and therefore should be accepted a priori?". We are interested in a
specific rule, 
namely NC which informally speaking states that
``A proposition of the form {\em All $F$ are $G$} is supported by
the observation that a particular object is both $F$ and $G$"
\cite{hempel45}. 
Does the fact that NC, 
does not seem to be
counterintuitive suffice to persuade us that it is a plausible rule
of induction? How can we be sure that it does not violate other
intuitively acceptable rules and principles? As the notorious \emph{raven
paradox} \cite{hempel45} shows, NC actually does entail
counterintuitive consequences, and more than seven decades of
discussion about this paradox shows that assessment of rules of
induction can be extremely problematic.

A summary of the \emph{raven paradox} is as follows: The hypothesis $H$ := ``\emph{All ravens are black}" is logically equivalent to $\hat{H}$ := ``\emph{Every thing that is not black is not a raven}". A \emph{green apple} is neither black nor a raven therefore according to NC its observation should confirm $\hat{H}$. But $H$ is logically equivalent to $\hat{H}$, so we end up with  a Paradoxical Conclusion (PC), that an observation of a \emph{green apple} confirms that all ravens are black, which is counterintuitive. 
In order to resolve the paradox, either it should be shown that NC is not a plausible rule of induction or it should be claimed that PC holds and should not be considered as being counterintuitive.\footnote{Some authors have even denied the equivalence of $H$ and $\hat{H}$ \cite{scheffler72}.} 

In order to study the paradox from a Bayesian perspective, first we make a distinction between (objective) background knowledge by which we exclusively refer to the knowledge that can be represented (and consequently can be thought of) as previously observed events, and any other kinds of information which we consider as being subjective and a property of the a priori chosen probability measure (i.e.\ the initial degrees of beliefs).
The cogency of inductive rules is significantly affected by the given background knowledge and the chosen measures.    
For example, it is already known that relative to some background knowledge, NC violates intuition (e.g.\ see \cite{good67}). In Section \ref{subsect:background}, we argue that relative to unrestricted background knowledge, not only NC but any rule of induction can be refuted.
Hempel himself believed that NC and the raven paradox should be considered in the context of absolutely no background information \cite{hempel67}. From a Bayesian perspective, however, this does not solve the problem. The reason is that background knowledge and priors are convertible (in the sense that they can produce the same effects). 
For example if we are not allowed to consider NC in the context that we possess the background knowledge that ``an object $a$ has a property $F$" (denoted as: $F_a$), we can (approximately) produce the same situation by subjectively believing that the probability that that object has the property $F$, is sufficiently close to 1 (denoted as: $pr(F_a) \approx 1$).\footnote{
According to \emph{Cournot's principle} \cite{cournot43}, 
we do not allow 1 (or 0) priors for events that may or may not happen. If it was allowed then assuming $pr(F_a)=1$, would exactly produce the same effect that possessing the background knowledge ``$a$ is $F$" would. This would be problematic, because assigning probability 1 to the events that are not determined (by background knowledge) may lead to undefined conditional probabilities (in case the complement events occur). However, assigning a probability that is arbitrarily close to one (i.e.\ $pr(F_a) \approx 1$) does not cause such a problem while it approximately produces the effect of the same background knowledge to arbitrary precision. 
}$^{,}$\footnote{
Although for Hempel and his contemporaries, the confirmation theory was (more or less) a \emph{logical} relation, akin to deductive entailment, rather than a \emph{probabilistic} relation (in the formal sense) \cite{fitelson10}, what we mentioned about the convertibility of objective information and subjective beliefs, was in a way reflected in their discussions: 
Good's \emph{Red Herring} \cite{good67} provided a hypothetical objective background setting with respect to which, NC does not hold. Hempel's assertion that ``NC should be considered in the context of no (objective) background information" was in fact an attempt to address such an issue. However, nothing could prevent Good from producing the same effect by simply replacing the objective information with subjective a priori beliefs of a new born baby (\emph{Good's baby} \cite{good68}).    
}
Therefore, if we want to restrict ourselves to the context of \emph{perfect ignorance}, not only should we possess no objective knowledge, we should also only be permitted to reason based on an absolutely unbiased probability measure. This raises an important question: ``What is an unbiased measure?" Due to its subjective nature, this question does not have a definitive answer but by far, the most widely considered choice is the \emph{uniform measure}.\footnote{
The main justification is due to the \emph{principle of maximum entropy} \cite{jaynes03}. This principle recommends that one choose, among all the probability measures satisfying a constraint, the measure which maximizes the \emph{Shannon entropy}. In the absence of any background knowledge (i.e.\ no constraint), the uniform probability measure maximizes entropy.
}$^{,}$\footnote{
An alternative to the \emph{uniform measure} is Solomonoff's \emph{theory of universal inductive inference} \cite{solomonoff64} which mathematically formalizes and puts together \emph{Occam's razor} (the principle that the simplest model consistent with the background knowledge should be chosen) and \emph{Epicurus' principle of multiple explanations} (that all explanations consistent with background knowledge should be kept) (see \cite{hutter07} or \cite{rathmanner11}). 
}        
On the other hand, it is well known that using the uniform measure, inductive learning is not possible \cite{carnap50}.
This shows that from the subjective probabilistic perspective, choosing a probability measure that satisfies the \emph{condition of perfect ignorance} and allows inductive learning, is arguably impossible.

It is also notable that demonstrating that a specific probability measure does (or does not) comply with a rule of induction (or a statement such as PC), does not illuminate the reason why a typical human observer believes that such a rule (or statement) is implausible (or plausible). Conversely, one might argue that compliance of a probability measure with a counterintuitive statement such as PC, may suggest that this measure does not provide a suitable model for inductive reasoning.
As an example, consider the following two works: \cite{maher99} and \cite{maher04}. They are among the most famous answers to the raven paradox. Using Carnap's measure \cite{carnap80}, in 1999 Maher argued that both NC and PC hold. In 2004 he suggested a more complex measure that led to opposite results: Maher showed that for this latter measure, neither NC nor PC holds in all settings.
Although Maher's works successfully show that at least for one probability measures NC holds and for another probability measures it does not, they do not show whether NC and PC are generally plausible or not.

The mainstream contemporary Bayesian solutions are not restricted to a particular measure and in this sense, are more general.  
According to \cite{fitelson06} almost all of them accept PC and argue that observation of a \emph{non-black non-raven} does provide evidence in support of $H$; however, in comparison with the observation of a \emph{black raven}, the amount of confirmation is very small. This difference, they argue, is due to the fact that the number of ravens is much less than the number of non-black objects.
However as \cite{vranas04} explains, these arguments have only been able to reach their intended conclusion by adding some extra assumptions about the characteristics of the chosen probability measure. He shows that the standard Bayesian solution relies on the almost never explicitly defended assumption that ``the probability of $H$ should not be affected by evidence that an object is \emph{non-black}." -- a supposition that he believes, is implausible, i.e.\ may hold or not. 

To summarize the above discussion: the general plausibility of a rule of induction cannot be determined if we restrict our study to particular (objective) background knowledge or a particular  probability measure. On the other hand, no rule of induction holds in the presence of an unrestricted choice of background knowledge and probability measure.
We conclude that rules of induction should be studied for different classes of background knowledge and priors. 
If a rule of induction holds relative to a large class of \emph{reasonable background knowledge} (i.e.\ information similar to our actual configuration of knowledge obtained from observations that often take place in real life) and relative to \emph{reasonable probability measures} (i.e.\ measures that have intuitively reasonable characteristics e.g.\ comply with other rules of induction which are more directly justified by our intuitive notion of induction), then we can claim that the studied rule is plausible, otherwise we cannot.

In this paper, we study NC with such an approach.  
In Section~\ref{sect:nc}, we present a formal representation for three rules of induction, namely, \emph{projectability} (PJ), \emph{reasoning by analogy} (RA) and \emph{Nicod's condition} (NC). We also define the form of background knowledge that is studied throughout the paper. Informally speaking, we only study pieces of knowledge that do not link the properties of one object to another object. They can be though of as knowledge that can be gained directly by observing the properties of some distinct objects. 
For example, the background knowledge: ``if object $a$ is a \emph{raven}, then object $b$ is not a \emph{raven}", is not of this form.  
While one can easily constitute pieces of information that do not have such a form and violate the aforementioned rules of induction, we have not found any illuminating counterexample to the assumption that relative to a piece of information that does not link properties of distinct objects together, PJ and RA comply with intuition. In the case of NC, we are more inquiring. In the next two sections, we study the restrictions of the probability measures that guarantee the validity of NC relative to two more specific background configurations (that can be expressed in the mentioned form). 
In Section~\ref{sect:setting1}, we find some sufficient conditions for the validity of NC and some sufficient conditions for its invalidity, relative to information about the kind (i.e.\ being raven or not) and color (i.e.\ being black or not) of some objects. The sufficient condition that we present for the validity of NC is less restrictive. 
However, this is insufficient for claiming that in this setting NC is generally plausible. 
Section~\ref{sect:setting2} deals with the setting where the exact number of objects having one property is known. For example we know how many ravens (or how many non-black objects) exist. We show that in this setting, measures that comply with Nicod's condition, do not always comply with PJ which seems to be the simplest formalization of \emph{inductive inference}. It is also shown that in the case of contradiction, intuition (arguably) follows PJ rather than NC. We think that this result is both interesting and somewhat surprising 
and should be considered as a main contribution of this paper.

One limitation of our basic setup is that it limits us to a universe with an arbitrary but known size. However, in Section \ref{sect:universe}, this strong assumption is replaced by the weaker assumption that there is a probability distribution over the possible sizes of the universe and this distribution is not affected by an observation of a single object. We show that under this weaker assumption,  the results from the former sections remain valid. 

In Section \ref{sect:conclude}, we summarize the paper and conclude that there is a tension between NC and our intuitive notion of inductive reasoning. We also suggest that \emph{reasoning by analogy} provides a viable alternative to formalize the seemingly intuitive statement that ``the observation that a particular object is both $F$ and $G$ confirms the hypothesis that any object that is $F$ is also $G$" without suffering form the shortcomings of NC. All theorems are proven in Section~\ref{sect:proofs}.       
\subsection{Notation, Basic Definitions and Assumptions} \label{sect:notate}
Throughout sections \ref{sect:nc} to \ref{sect:setting2} we work with a first-order language $L$ whose only nonlogical symbols are a pair of monadic predicates $F$ and $G$ and a set of constants $U$ (officially shown as) $\{ u_1, u_2, \ldots, u_\usize \}$ where {$\usize$} is a known positive integer. However, for simplicity we drop ``$u$" and refer to each constant by its index. 
We rely on the \emph{domain closure axiom} \cite{reiter80}, that is:
\begin{equation}\label{eq:axiom} 
\forall x \quad (x=1) \vee (x=2) \vee \ldots \vee (x = \usize)
\end{equation} 
where $1$ to $N$ are distinct constants i.e.\ $(1 \neq 2) \wedge (1 \neq 3) \wedge \ldots$.
Clearly, models of this axiom are restricted to interpretations with \emph{domains} containing exactly $\usize$ distinct individuals (objects) each of which is denoted by a constant in $U$. Using   this  bijection between the elements of the domain and constants, we refer to $U$ as the universe. 
  
Negation, conjunction, disjunction and material implication are 
respectively represented by 
``$\neg$", ``$.$" (or ``$\wedge$"), ``$\vee$" and ``$\rightarrow$". 
If $b$ is an individual and $\psi$ is a 1-place predicate (either atomic or a sentential combination of atomic predicates), $\psi_b$ is defined as a proposition that involves predicate $\psi$ and indicates: 
``$b$ has (or satisfies or is described by) $\psi$". 
Conjunction of several propositions 
$\psi_{m},\psi_{{m+1}},\ldots,\psi_{n}$ is abbreviated by $\psi_{m:n}$.
By definition, for $n < m$, $\psi_{m:n} := \top$ (i.e.\ tautology) and  $\psi_{m:m}:=\psi_{m}$.
More generally, if $b_m$ to $b_n$ are some objects (not necessarily consecutive), $\psi_{b_m : b_n} := \psi_{b_m} \! \wedge \ldots \wedge \psi_{b_n}$. 
The general hypothesis $H := (\forall x \;\, \impl_x)$ (which is equivalent to $\impl_{1:{\usize}}$) where $\impl_{b} := F_{b} \rightarrow G_{b}$. We also let: 
\begin{equation}\label{eq:q1234}
\Qff_{ b }:= \neg F_b.\neg G_b; \quad 
\Qft_{ b } := \neg F_b.G_b; \quad 
\Qtf_{ b } := \neg (\impl_b) = F_b.\neg G_b; \quad  
\Qtt_{ b } := F_b.G_b
\end{equation}
Any of $\Qff_{ b }$ to $\Qtt_{ b }$ 
defined by relation (\ref{eq:q1234}) 
is referred to as a \emph{complete description} of an object $b$.\footnote{
$\Qff$, $\Qft$, $\Qtf$ and $\Qtt$  are what Maher calls $Q_4$, $Q_3$, $Q_2$ and $Q_1$ respectively.
What we call \emph{complete description}, he calls \emph{sample proposition}
 \cite{maher99}.
}
For example if $F$ and $G$ represent ``ravenhood" and ``blackness" properties respectively, then $\Qff_{b}$ means ``$b$ is not a raven and is not black" 
and so on. $\impl_{b}$  means ``if $b$ is a raven then it is black" and $H$ is the general hypothesis that ``for all $b$, if $b$ is a raven then it is black". Clearly all complete descriptions which provide counterexample to $H$ are in the form $\Qtf_{b}$.

We define $\Delta$ as the set of all propositions $\rho$ which are in the following form (or by simplification can be converted to it): 
\begin{equation} \label{eq:artless}
\rho := \psi^{{1}}_{b_{1}} . \psi^{{2}}_{b_{2}} . \ldots \psi^{{k}}_{b_{k}} =  \bigwedge_{x = 1}^{k}  \psi^{x}_{b_{x}}
\end{equation}
where $\psi^1$ to $\psi^k$ are some predicates in $\{F$, $\neg F$, $G$, $\neg G$, $\Qff$, $\neg \Qff$, $\Qft$, $\neg \Qft$, $\Qtf$, $\impl$, $\Qtt$, $\neg \Qtt \}$
and $b_1$ to $b_{k}$ are some mutually distinct objects: $\{ b_{1}, b_{2}, \ldots b_{k} \} \subseteq U$.
Note that $\Delta$ is in fact the set of all propositions that do not link the properties of different objects together.   
We define the set of all individuals described by $\rho$ as: $\inds_{\rho} := \{ b_{1}, b_{2}, \ldots b_{k}\}$. We refer to the set of (simple) predicates involved in $\rho$ by: 
$ \preds_{\rho} := \{ \psi^{1}, \ldots , \psi^{k} \}$. 
For example, the proposition $\rho' := \Qtt_{1} \vee \Qtf_{3}$ is not in $\Delta$ but $\rho'' := \Qtt_{1} . \Qtt_{3} . \impl_{4} \in \Delta$ (assuming $\usize \geq 4$). $\inds_{\rho''} = \{1, 3, 4\}$ and  $\preds_{\rho''} = \{\Qtt, \impl \}$. 
By definition, empty (or tautologous) proposition $\top$ is a member of $\Delta$ with $\inds_{\top} = \emptyset$. 
Two subsets of $\Delta$ are defined as follows:
\begin{align*}
\delta &:= \big \{d \in \Delta : \preds_d = \{ \Qff, \Qft, \Qtt \}       \big \}\\
\Omega &:= \big \{ d \in \Delta : \inds_d = U, \preds_d = \{ \Qff, \Qft, \Qtf, \Qtt \}   \big \}
\end{align*}
Informally speaking, $\delta$ is the set of propositions that \emph{completely describe} some individuals and do not falsify $H$. 
Likewise, $\Omega$ is the set of propositions that \emph{completely describe} all objects of the universe. We refer to any member of $\Omega$ as a \emph{Complete Description Vector} (CDV). 
Note that each CDV corresponds to a unique \emph{model} or \emph{world} (up to isomorphism). In other words, every interpretation that makes a CDV (and aforementioned axiom (\ref{eq:axiom})) true, uniquely determines the value of any sentence in $L$,\footnote{
The proof is straightforward. With respect to the domain closure axiom, all quantifiers are bounded. Therefore all sentences are convertible to quantifier-free forms and consequently convertible to full disjunctive normal form (DNF) which is in fact a disjunction of some CDVs. In any world, only one CDV is true,  therefore only sentences containing that CDV (when expressed in full DNF) are true.  
} therefore we can consider them as (representatives of) different worlds.
The probability measures which we are concerned with, are 
defined over the sample space $\Omega$ 
with the power set as $\sigma$-algebra. 
No other restriction is imposed on the choice of measure 
unless it is mentioned explicitly. 
For each proposition $\rho$, let $\omega_\rho^{\Omega} := \{ o \in \Omega : o \models \rho \}$ be the set of all CDVs that entail $\rho$. 
We say that the event $\omega_\rho^{\Omega}$ \emph{corresponds} to the proposition $\rho$ (and vice versa).
For convenience' sake, except in Section \ref{sect:universe}, 
we represent the probability of events by the probability of their \emph{corresponding propositions}\footnote{
In Section \ref{sect:universe}, we simultaneously deal with more than one sample space. While w.r.t.\ different sample spaces, propositions may correspond to different events, in that section we directly represent probability events by their relevant sample space subsets.
}; formally, for propositions $\rho$, we let $pr(\rho) := pr(\omega_\rho^{\Omega})$.
According to \emph{Cournot's principle} \cite{cournot43}, 
we do not allow 1 (resp. 0) priors to the sentences that are not valid (resp. unsatisfiable). 

Objective background knowledge (or simply background knowledge) is what we are certain about and can be represented by a subset of $\Omega$. 
The more formal definition of the background setting studied throughout this paper and its corresponding restrictions are given in Section \ref{subsect:background}.

We equate ``inductive support" with ``probability increment": It is said that in the presence of background knowledge 
$D$, evidence $E$ \emph{confirms} hypothesis $H$ iff:
$$
pr(H | E . D ) > pr ( H | D)
$$
We are only interested in the case where $E$ and $D$ are consistent, $pr(E)$ and $pr(D)$ are 
positive and $E$ is not determined by $D$, i.e.\ $0<pr(E|D)<1$. 
\section{Inductive Reasoning and Nicod's Condition} \label{sect:nc}
The fundamental assumption behind \emph{inductive inference} is the so-called \emph{principle of the uniformity of nature} \cite{hume88} (or \emph{the immutability of natural processes} \cite{popper59}) based on which, uniformity and trend are more probable than diversity and anomaly a priori.
Let us assume that $\Delta$ is a set of background knowledge configurations for which we ``intuitively" expect that inductive inference holds (for more discussion refer to Section \ref{subsect:background}). Relative to pieces of information in $\Delta$, we present the following varieties of inductive inference (i.e.\ \emph{inductive rules}):\vspace{2.5mm}\\
{\bf Projectability.}
For all objects $a$ and $b$ and background knowledge $D \in \Delta$ that does not determine $\psi_{a}$ or $\psi_{b}$, based on  \cite{maher04}, one (and apparently the simplest) kind of inductive inference, namely \emph{projectability}\footnote{
According to \cite{carnap50} \emph{predictive inference} (i.e.\ inference from a sample to another sample) is the most important kind of inference and the most important special kind of it, \emph{singular predictive inference}, is inference from a sample to an individual object. Maher's \emph{projectability} is in fact a special kind of \emph{singular predictive inference}:  inference from one individual to another individual.   
}$^{,}$\footnote
{Maher's original relation does not mention 
background knowledge, and only deals with \emph{strong projectability} (which he calls \emph{absolute projectability}).
}, is defined as follows:
\begin{align}
\text{Strong projectability:} \quad \forall a , b  \in U            \quad pr(\psi_{b} | \, \psi_{a} .D) > pr(\psi_{b} | D) \notag\\
\label{eq:weak-proj} 
\text{Weak projectability \!\! (PJ):} \quad \forall a , b \in U \quad pr(\psi_{b} | \, \psi_{a} .D) \geq pr(\psi_{b} | D)
\end{align}
Projectability (relative to predicate $\psi$) can be justified as follows: The evidence $\psi_{a}$ increases the proportion of the
observed individuals that have the predicate $\psi$. Thus, according to the principle of the uniformity of nature, the estimated frequency of the predicate $\psi$ in the total population should also be increased because the uniformity between the characteristics of the sample and the total population is considered to be likely.\vspace{2.5mm}\\
{\bf Reasoning by Analogy (RA).}
The observation that two individuals have some common properties, increases the probability that their unobserved properties are also alike, because it is likely that there is a uniformity between the characteristics of unobserved properties and the observed ones. Maher has formalized one variation of \emph{reasoning by analogy} (or \emph{inference by analogy} \cite{carnap50}) as: $\forall a, b \in U \;\; pr(G_b |F_b . \Qtt_{a}) > pr(G_b | F_b)$ \cite{maher04}. We generalize the relation to cover the case where background knowledge $D \in \Delta$ (that does not determine the value of $F_a$, $G_a$ and $G_b$) is also present:
\begin{equation} 
\mbox{Reasoning\! by\! analogy \!\!\! (RA):\,} 
\forall a, b \in U \quad pr(G_b | F_b . \Qtt_{a} . D) > pr(G_b | F_b . D)\! \label{eq:analogy}
\end{equation}
{\bf Nicod's Condition (NC).}
For $U := \{ 1, \ldots, {\usize}\}$, $H := \impl_{1:{\usize}}$, all $a \in U$ and $D \in \Delta$ that does not determine the value of $\Qtt_{ a }$ or $H$, we say NC holds for $D$ iff:
\begin{equation}\label{eq:nc}
\mbox{Nicod's Condition (NC):} \hspace{0.3cm}  
pr(H| \Qtt_{ a } . D) > pr( H | D )
\end{equation}
NC is stronger than PJ or RA in the sense that it deals with the confirmation of a generalization rather than a singular prediction. In other words, NC is a form of \emph{enumerative induction} but PJ and RA are forms of \emph{singular predictive inference}.
\subsection{Restrictions on the Background Knowledge}  \label{subsect:background}
Obviously, relative to unconstrained background knowledge, no rule of induction holds in general. 
For example, in the presence of background knowledge $D' := F_{1} \rightarrow (\neg F)_{2:{\usize}}$, at least relative to evidence $F_{1}$, PJ does not hold. 
Similarly, (as \cite{maher04}, Theorem 12 formally shows), in the presence of background knowledge $D'' := \Qtt_{ 1 } \rightarrow \neg H$, NC does not hold 
(for evidence $\Qtt_{ 1 }$).

To prevent such problems, the biggest set of background configurations studied through out this paper is $\Delta$\footnote{
Note that we do not claim that no background knowledge that is not a member of $\Delta$ is not plausible. Investigation of rule of inductions relative to such knowledge, is simply beyond the scope of this paper.  
},
which
according to its definition in Section \ref{sect:notate}, is the set of all consistent propositions that can be expressed 
in the form of a conjunction of some propositions that involve
 $F$, $G$, $\Qff$, $\Qft$, $\Qtf$, $\Qtt$ or their negations.

Obviously, each member of $\Delta$ can be expressed in the form of a conjunction of some propositions each of which describes only one individual. Consequently, problematic statements that interlink properties of different individuals are not expressible.
As an example, the mentioned pathological examples $D'$ and $D''$ are not in $\Delta$.

In the case of PJ and RA, we did not find a pathological example in $\Delta$, relative to which, the rule of induction contradicts intuition. However, in the case of NC it is already claimed that relative to background knowledge $D''' := \neg F_a \in \Delta$, it is not intuitively sound to expect that the evidence $\Qff_{a}$ confirms $H$ \cite{fitelson10}.    
In Sections \ref{sect:setting1} and \ref{sect:setting2}, we will investigate the validity of NC relative to two interesting subsets of $\Delta$. 
\subsection{Restrictions on the probability measure.}
In Section \ref{sect:setting1} (Setting 1), we impose no restriction on the choice of the probability measure but in Section \ref{sect:setting2} (Setting 2), we assume that the probability measure is  \emph{exchangeable} \cite{carnap80} in a sense that probabilities are not changed by permuting individuals (i.e.\ swapping the name of objects). To introduce this restriction formally, we need the following definitions:
\begin{definition}
By the term \emph{permutation}, we always refer to a bijection from a set of all objects $U$ to itself. Throughout this paper, we denote any arbitrary permutation by $\permute$ (or $\permute'$ and $\permute''$ when we deal with more than one permutation). 
Having a proposition $\rho$, the proposition $\rho^\permute$ is obtained from $\rho$ by replacing any occurrence of any individual $b$ with $\permute(b)$.
\end{definition}
\begin{example}
If $U := \{1, 2, 3\}$, the function $\permute : U \rightarrow U$ defined by $\permute(1) = 1$; $\permute(2) = 3$ and $\permute(3) = 2$, is a permutation with a fixed point $1$. 
For short we write $\permute := \{2 / 3 ; 3 / 2\}$. If we define $\rho$ := $F_{1} \vee G_{3}$, then $\rho^\permute = F_{1} \vee G_{2}$.  
\end{example}
\begin{definition}[Exchangeability.]
The probability measure $pr$ is \emph{exchangeable} if  
for all propositions $A$ and $B$ and all permutations $\permute$, $pr(A|B) = pr(A^\permute | B^\permute)$. 
\end{definition}
\section{Validity of NC when Background Knowledge Consists of Complete Descriptions of Some Individuals (Setting 1)} \label{sect:setting1}
In Section \ref{sect:notate}, $\delta$ was defined as the set of all background knowledge that do not refute $H$ and describe some individuals completely (e.g.\ in the case of the raven paradox, members of $\delta$ represent the knowledge that we are already aware of the color and kind (i.e.\ the state of being raven) of some individuals and none of these known objects have been a non-black raven). 
Clearly, $\delta \subset \Delta$.
Theorem \ref{theo:a2-a3-nc} shows that if the chosen probability measure satisfies some conditions, then for any background knowledge $D \in \delta$, NC holds (for predicates $F$ and $G$).
On the other hand, Theorem \ref{theo:a1-nc} shows that under alternative conditions, for some $D \in \delta$, NC does not hold.
\begin{theorem} \label{theo:a2-a3-nc} If a probability measure complies with the following relation:
\begin{equation} \label{eq:bq1}
\forall B \in \Delta, \forall a, b \notin \inds_B 
\qquad pr(\Qtf_{ b } | \Qtt_{ a } . B) \leq pr(\Qtf_{ b } | B) 
\end{equation}
 then, for this measure and any $D \in \delta$ that does not determine $\Qtt_{ a }$ or $H$, NC holds, i.e.\ relation (\ref{eq:bq1}) entails: $pr(H | \Qtt_{ a }. D) > pr( H | D)$.
\end{theorem}
\begin{example} 
If background knowledge consists of complete descriptions of some individuals, by Theorem \ref{theo:a2-a3-nc} for all pairs of predicates $F$ and $G$, the uniform measure complies with NC since regardless of the interpretation of $F$ and $G$,  for this measure, $\forall B \in \Delta$ \& $\forall a, b \notin \inds_B \quad pr (\Qtf_{ b } | \Qtt_{ a } . B) = pr(\Qtf_{ b } | B)$.
This is not surprising since using this measure, learning is impossible (see \cite{carnap50}). This means that no observation changes the probability of being $\Qtf$ for an unobserved object. Nonetheless, for this measure NC is valid because any evidence in the  form of 
$\Qff_{ a }$, $\Qft_{ a }$ or $\Qtt_{ a }$ 
confirms $H$ for the simple reason that it removes the possibility that the observed object (i.e.\ $a$) is a counterexample to $H$.  
\end{example}
\begin{example}
In Carnap's theory of inductive probability \cite{carnap80}:
\[ 
pr({\psi}_b | E ) = \frac{n_{\psi} + \lambda \cdot pr({\psi}_b )} {n + \lambda} 
\]
In the above relations, $n$ is the number of objects mentioned by evidence $E$; $n_{\psi}$ is the number of mentioned objects which satisfy predicate $\psi$, and $\lambda$ is a constant  measuring the resistance to generalization. Note that $b$ should not be mentioned by $E$, i.e.\ $b \not \in \inds_E$.  
Using this measure and choosing 
$\psi \in \{ \Qff, \Qft, \Qtf, \Qtt \}$, in the presence of background knowledge $D \in \delta$ such that $a, b \notin \inds_D$:
$pr(\Qtf_{ b } | \Qtt_{ a } . D)$ = $\frac{\lambda.
pr(\Qtf_{ b } | D)}{1 + \lambda} < pr(\Qtf_{ b } | D)$. Thus by Theorem \ref{theo:a2-a3-nc}, for the class of background knowledge in the form of conjunction of some 
$\Qff$, $\Qft$ and/or $\Qtt$ 
for distinct individuals, this measure complies with NC. This is equivalent to the setting chosen by \cite{maher99} and its corresponding results. 
\end{example}
\begin{theorem} \label{theo:a1-nc} 
If a probability measure complies with restrictions (\ref{eq:not-nc1}) and (\ref{eq:not-nc2}), 
then for this measure (and predicates $F$ and $G$) and background knowledge $D \in \delta$, NC does not hold.
\begin{align}\label{eq:not-nc1}
&\forall B \!\in\! \Delta, \forall a, b \notin \inds_B 
&&pr(\Qtf_{ b } | \Qtt_{ a } . B) > pr(\Qtf_{ b } | B) \\
\label{eq:not-nc2}
&\forall a \!\notin\! \inds_D 
&& pr(\neg \impl_a | \impl_{ {b_1} : \, {b_n} } . D) < 
pr(\Qtf_{ {b_1} } |\, \Qtt_{ a } . D) - pr(\Qtf_{ {b_1} } |\, D)
\end{align}
In the above relations $a \neq b$ and $b_1$ to $b_n$ represent an arbitrary enumeration of all individuals (except $a$) that are not mentioned by $D$ (that is, $\inds_D = U \backslash \{{b_1} \ldots {b_n}, a\}$). 

According to restriction (\ref{eq:not-nc2}), the probability that $a$ is not $\impl$ given that all other objects in the universe are $\impl$ should be less than the degree of confirmation by evidence  $\Qtt_{ a }$ of a hypothesis that an unobserved object ${b_1}$ is $\Qtf$. Note that $b_1$ can be the index of any unobserved object. 
\end{theorem}

\begin{example} 
\cite{maher04} proposes a measure based on the formula:
\begin{align*}
pr(\Qtf_{ b } | E) = pr(I) \cdot \frac{n_F+ \lambda \cdot pr(F_b)}{n + \lambda} 
\cdot \frac{n_{\overline{G}} + \lambda \cdot pr(\neg G_b)}{n+\lambda}
+ pr(\neg I) \cdot \frac{ n_{\Qtf} + \lambda \cdot pr(\Qtf_{ b } )} {n + \lambda}
\end{align*}
In the above relation, $n$ is the number of objects mentioned by $E$; 
$n_F$ and $n_{\overline{G}}$ denote the number of mentioned objects which are $F$ and $\neg G$ respectively.
\\In this expression, the prior probability of $\Qtf$ i.e.\ $pr(\Qtf_{ b })$ has to be equal to $pr(F_b)\cdot pr(\neg G_b)$ 
and $pr(I)$ and $\lambda$ are parameters. 
Maher proposes a counterexample for NC where 
$\usize=2$ (Let $U := \{ a, b \}$), $\lambda = 2$, $pr(I) = 0.5$ and 
prior probabilities are $pr(F_b)=0.001$ and $pr(G_b)=0.1$. This conclusion can be confirmed independently by Theorem \ref{theo:a1-nc} as follows: 
\\1.~For these parameters, the only member of $\Delta$ that does not contain $a$ and $b$, is $B=\emptyset$ for which relation (\ref{eq:not-nc1}) holds if $pr(F_b)<0.25$.
\\2.~By assuming: $\big(\,  pr(\Qtf_{ b } |\, \Qtt_{ a } ) - pr(\Qtf_{ b } ) \big) \geq 0.06$ and $pr(G_b) = 0.1$, a cumbersome calculation shows that (for empty background knowledge) relation (\ref{eq:not-nc2}) holds if: $pr(F_b)<0.0983$ which covers Maher's proposed configuration.  
\end{example}

Comparing Theorems \ref{theo:a2-a3-nc} and \ref{theo:a1-nc} shows that creating a probability measure that contradicts NC (w.r.t.\ $D \in \delta$) is harder than making a measure that  complies with it (for the same background setting) because the former measure has to satisfy more constraints.  
The reason is that even if in a measure, evidence $E:= \Qtt_{ a }$ does not affect the probability of $\Qtf$ for unobserved objects  (as in the case of the uniform distribution), every hypothesis that is not refuted by $E$ (including $H$) is confirmed by it since the observation has reduced the number of possible counterexamples by one. 
On the other hand, in the case of a measure that does not comply with NC, 
not only should $E$ confirm $\Qtf$ for unobserved objects, but the effect of this confirmation should be so substantial that it overwhelms the effect of the elimination of one counterexample to $H$.\footnote{
To see how the effect of elimination of one possible counterexample leads to relation~(\ref{eq:not-nc2}), refer to the proof of Theorem \ref{theo:a1-nc} in Section \ref{sect:proofs}.
}
However, in the case where the size of the universe is large, the latter effect should be minute. This is reflected in relation (\ref{eq:not-nc2}) as follows: 
If $\usize$ is large, then at least for measures that comply with projectability, $pr(\neg \impl_a | \impl_{ {b_1} : \, {b_n} } . D) \approx 0$, because if it is known that all objects in the universe except $a$ are $\impl$, then it should be quite probable that $a$ is $\impl$ too.
Therefore, even if the degree of confirmation of $\Qtf_{b_1}$ by evidence $\Qtt_{ a }$  (i.e.$\ pr(\Qtf_{ {b_1} } |\, \Qtt_{ a } . D) - pr(\Qtf_{ {b_1} } |\, D)$) is very small\footnote{
Note that by relation (\ref{eq:not-nc1}), this degree of confirmation is positive.
} , relation (\ref{eq:not-nc2}) holds.\footnote{
Here is another justification for the above argument: By definition, a probability measure defined over a first-order language with an infinite domain is \emph{Gaifman} iff the probability of the generalization of any predicate (in our case, $\impl$) is equal to the probability of the conjunction of some positive instances when their number tends to infinity \cite{gaifman82} or alternatively, 
$pr(\forall x \; \psi_x | \psi_{1:n}) \xrightarrow{n \rightarrow \infty} 1$ (see \cite{hutter13} thm. 27)
and consequently $pr(\neg \psi_a | \psi_{1:n}) \xrightarrow{n \rightarrow \infty} 0$.
Since the Gaifman condition is what we intuitively expect from \emph{generalization} over an infinite universe, it can be considered as a very simple and intuitive rule of induction. In our case, if the universe was infinite and the measure was assumed to be Gaifman, inequality (\ref{eq:not-nc2}) would always hold. However we have assumed that the universe is finite therefore we cannot remove this inequality. What we can say is that for very large domains, relation (\ref{eq:not-nc2}) is a very weak condition.  
}
To summarize:
\begin{itemize}
\item \emph{If regardless of the choice of background knowledge, an observation $F_a . G_a$ does not confirm that any unobserved individual is an $F$ that is not $G$, then relative to any background knowledge in $\delta$, NC holds.} 
\item \emph{If regardless of the choice of background knowledge, an observation $F_a . G_a$ confirms that any unobserved individual is an $F$ that is not $G$, and on the other hand, the effect of elimination of one counterexample via an observation is negligible, then relative to any background knowledge in $\delta$, NC does not hold.}
\end{itemize}
The above statements delegate the assessment of NC (a form of enumerative induction) to the assessment of expressions which deal with singular predictions. Hence, a new perspective on the nature of NC is provided: Should regardless of the interpretation of $F$ and $G$, (the observation of) an $F$ that is $G$ disconfirm that any unobserved object is $F$ but not $G$?  

For example, relative to background knowledge and a probability measure that reflect our actual configuration of knowledge, should the observation of an $F$=``walnut'' that is $G$=``round" decrease the probability that any unobserved object is a walnut but not round? Indeed yes; therefore by Theorem \ref{theo:a2-a3-nc}, in this case and for these predicates, NC holds.
Should the observation of an $F$=``round", G=``walnut" decrease the probability that any unobserved object is ``round"  but not a ``walnut"? Arguably not.\\
Should the observation of an $F$=``ogre" which is $G$=``old" decrease the probability that we might encounter an ogre which is not old? Definitely not! \footnote{
This confirmation asymmetry may be due to possible asymmetry in background knowledge and/or prior possibilities of different predicates. For example according to our actual configuration of knowledge, the prior probability of ``being an ogre" (for any individual) is quite low. This is a key point in the existing arguments: \emph{Good's baby}  \cite{good68} (that assigns low probability to \emph{ravenhood}) and Maher's \emph{unicorn} \cite{maher04}. But unlike our discussion, these arguments do not reduce the assessment of NC to a singular prediction.
}
Therefore, in this case, we are intuitively using a probability measure that satisfies the condition (\ref{eq:not-nc1}). 
Now assume that we have seen all objects of the world except one. 
It has happened that any observed object that has been an ogre has been old as well. 
Is it reasonable to believe that it is improbable that the last unobserved object is a young ogre? 
If yes, then our intuitive measure also complies with restriction (\ref{eq:not-nc2}), hence by Theorem \ref{theo:a1-nc}, by this denotation for $F$ and $G$, plausible probability measures do not comply with Nicod's condition.
\section{NC vs. PJ when the Number of Objects having One Predicate is Known (Setting 2)} \label{sect:setting2}

This section studies NC in the presence of a completely different background setting where
we know that exactly $k$ individuals are $F$ (e.g.\ ravens) and the rest are not $F$, but we do not know anything about the other property (e.g.\ their color).

First we focus on a simpler setting where we know exactly which objects are $F$ and which objects are not $F$ (e.g.\ we know that objects $1$ to $k$ are $F$ and the rest of the universe i.e.\ objects ${k+1}$ to ${\usize}$ are not $F$).

\begin{theorem} \label{theo:wason}
For $U := \{1, 2, \ldots {\usize}\}$ and $D := F_{1 : k}. (\neg F)_{{k+1}:{\usize}}$, weak projectability (PJ) entails:
\begin{align}
pr(H \,| \, G_{k} \,.\, D) &> pr(H| D)                   \label{eq:7.1}\\
pr(H \,| \, G_{{\usize}} \,.\, D) &\geq pr(H| D)        \label{eq:7.2}\\
pr(H \,| \, \neg G_{{\usize}} \,.\, D) &\leq pr(H| D) \label{eq:7.3}
\end{align}
and reasoning by analogy (RA) entails:\footnote{
Therefore, in this setting both PJ and RA suggest that $H:=\forall b \; \impl_{b}$ is confirmed by evidence $G_{k}$ but RA provides no answer whether or not evidence $G_{N}$ should confirm (or disconfirm) $H$. The reason is that (as the proof of the theorem which is provided in Section \ref{sect:proofs} shows) in the presence of background knowledge $F_{1 : k}. (\neg F)_{{k+1}:{\usize}}$, validity of $H$ only depends on property $G$ of objects $1$ to $k$ that do not have a common property with object $N$.    
}
\begin{align}
pr(H \,| \, G_{k} \,.\, D) &> pr(H| D)                   \label{eq:ra.7.1}
\end{align}
\end{theorem}

Next, we show that these results are valid in the general setting where the background knowledge is such that we only know the exact number of objects being $F$ but we do not know their names. In other words, we know that exactly one  combination of $k$ out of $N$ objects of the universe are $F$ but we do not know which combination. But before that, we should formalize such knowledge in the form of an event (i.e.\ a subset of the sample space).   

\begin{definition}
$\mathfrak{C}_{U,k} := \{ C : C \subseteq U  , |C|=k \}$ is defined as the set of all (distinct) subsets of $U$ which contain exactly $k$ individuals. 
 Obviously, the cardinality of $\mathfrak{C}_{U,k}$ is ${\usize\choose{k}}$.
\end{definition}

\begin{example} \label{ex:pi}
Given $U := \{1, 2, 3, 4 \}$, $\mathfrak{C}_{U,2}$ = $\big\{ \{1, 2\}$, $\{1, 3\}$, $\{1, 4\}$, $\{2, 3\}$, $\{2, 4\}$, $\{3, 4\} \big\}$.
\end{example}
\begin{definition}
For $1 \leq k \leq \usize$, ``Exactly $k$ objects of the universe $U$ are $F$'' is formally defined as follows: 
\begin{equation} \label{eq:exact_def}
\textsc{Exact}(k, U, F) := 
\bigvee_{C \in \mathfrak{C}_{U,k}}
\big(
\bigwedge_{b' \in C} \! F_{b'} \;.
\bigwedge_{b'' \not\in C} \!\! \neg F_{b''}
\big)
\end{equation}

\end{definition}

\begin{example} \label{ex:exact} 
In the previous example, $\textsc{Exact}(2, U, F) = (F_{1} . F_{2} . \neg F_{3} . \neg F_{4})$ $\vee$
$(F_{1} . F_{3} . \neg F_{2} . \neg F_{4})$ $\vee$
$(F_{1} . F_{4} . \neg F_{2} . \neg F_{3})$ $\vee$
$(F_{2} . F_{3} . \neg F_{1} . \neg F_{4})$ $\vee$\\
$(F_{2} . F_{4} . \neg F_{1} . \neg F_{3})$ $\vee$
$(F_{3} . F_{4} . \neg F_{3} . \neg F_{4})$.
\end{example}
By comparing definition (\ref{eq:exact_def}) with the definition of $\Delta$, it becomes clear that for $1<k<\usize$, $\textsc{Exact}(k, U, F) \not \in \Delta$, therefore we do not expect that in the presence of such background knowledge, rules of induction hold in general and they actually don't. 
For instance, knowing that exactly $k$ objects are $F$, the evidence that a particular object is $F$, confirms that any other object is not $F$,\footnote{
Suppose that you are in a camp populated by 100 captives, and it is known that 10 of them will be chosen randomly to be executed; Whenever someone except you is chosen, it is reasonable to be  more optimist about your fate, for the simple reason that $\frac{9}{99} < \frac{10}{100}$.
} which contradicts PJ:
\begin{equation*}
\text{(intuitively):} \;\; \forall b \neq a \in U \;\;  pr \big(F_b | F_a . \textsc{Exact}(k, U, F) \big) < pr \big( F_b | \textsc{Exact}(k, U, F) \big) 
\end{equation*}
 
However, the following theorem shows that for the hypothesis that we are interested in i.e.\ $H:=\forall b \; \impl_{b}$, the background knowledge $\textsc{Exact}(k, U, F)$
is equivalent to $F_{1 : k} .\neg  F_{{k +1} : {\usize}}$ which is a member of $\Delta$. Therefore, in the case of the raven paradox and background knowledge $\textsc{Exact}(k, U, F)$, the rules of induction (that are assumed to hold relative to background knowledge in $\Delta$) should still hold.    
\begin{theorem} \label{theo:exact2list}
If $U = \{ 1, \ldots, {\usize} \}$ and $a$ is an arbitrary member of $U$ and assuming that a probability measure $pr$ is exchangeable:
\begin{align}
pr \big( H | \, \textsc{Exact}(k, U, F) \big)                                   &= pr \big(H | \, F_{1 : k} .\neg  F_{{k +1} : {\usize}} \big) \label{eq:exact2list}\\
pr \big( H | \, \textsc{Exact}(k, U, F) .         F_a.         G_a \big) &= pr \big(H | \, F_{1 : k} .\neg  F_{{k +1} : {\usize}} . G_{k} \big) \label{eq:exact2list_fa_ga}\\
pr \big( H | \, \textsc{Exact}(k, U, F) . \neg F_a.         G_a \big) &= pr \big(H | \, F_{1 : k} .\neg  F_{{k +1} : {\usize}} . G_{{\usize}} \big) \label{eq:exact2list_fa_notga}\\
pr \big( H | \, \textsc{Exact}(k, U, F) . \neg F_a. \neg G_a \big) &= pr \big(H | \, F_{1 : k} .\neg  F_{{k +1} : {\usize}} . \neg G_{{\usize}} \label{eq:exact2list_notfa_notga} \big) 
\end{align}
\end{theorem}

The formal proof of this theorem is presented in Section \ref{subsect:premute_proof}, but the following simple example shows the main idea behind the general proof.

\begin{example}
Having $U:=\{1, 2, 3\}$, we show that:\\ 
$pr \big(H | \, \textsc{Exact}(2, U, F) . F_{3} . G_{3} \big)$ $=$ $pr \big(H | F_{1} . F_{2} . \neg  F_{3} . G_{2} \big) $
 (that is relation (\ref{eq:exact2list_fa_ga}) for $a := 3$ and $k := 2$) as follows:
\begin{align*}
&pr \big(\textsc{Exact}(2, U, F) . F_{3}.  G_{3} | H \big)  \\
&= 
pr \big( (
F_{1} . F_{2} . \neg F_{3} \vee 
F_{1} . F_{3} . \neg F_{2} \vee 
F_{2} . F_{3} . \neg F_{1})
. (F_{3} . G_{3}) | \impl_{1} . \impl_{2} . \impl_{3} \big), \text{by def.} \\
&=
pr(F_{1} . F_{3} . \neg F_{2} . G_{3} \vee 
F_{2} . F_{3} . \neg F_{1} . G_{3} | \, \impl_{1} . \impl_{2} . \impl_{3}), \text{by simplification}\\
&=
pr( F_{1} . F_{3} . \neg F_{2} . G_{3} | \, \impl_{1} . \impl_{2} . \impl_{3}) +
pr( F_{2} . F_{3} . \neg F_{1} . G_{3} | \, \impl_{1} . \impl_{2} . \impl_{3}), \\ 
&\hspace{1cm}\text{by $\sigma$-additivity of disjoint events (3rd Kolmogorov probability axiom)} \\
&=
pr( F_{1} . F_{2} . \neg F_{3} . G_{2} | \, \impl_{1} . \impl_{3} . \impl_{2}) +
pr( F_{1} . F_{2} . \neg F_{3} . G_{2} | \, \impl_{3} . \impl_{1} . \impl_{2}), \\
&\hspace{1cm}\text{by exchangeability assumption, using premutation $\permute' := \{3/2; 2/3\}$}\\
&\hspace{1cm} \text{on the first term and $\permute'' := \{3/2; 1/3; 2/1\}$ on the second term} \\
&=
2 \cdot pr( F_{1} . F_{2} . \neg F_{3} . G_{2} | H)
\end{align*}
Similarly it can easily be shown that:\\ $pr \big( \textsc{Exact}(2, U, F) . F_{3}.  G_{3} \big) = 2 \cdot pr \big( F_{1} . F_{2} . \neg F_{3} . G_{2} \big)$.
Therefore by Bayes rule:
\begin{multline*}
pr \big( H | \textsc{Exact}(2, U, F) . F_{3} . G_{3} \big) = \frac{pr\big( H \big) \cdot pr \big( \textsc{Exact}(2, U, F)|H \big)}{pr \big( \textsc{Exact}(2, U, F) \big)}\\
= \frac{2 \cdot pr(H)  \cdot pr( F_{1} . F_{2} . \neg F_{3} . G_{2} | H)}{2 \cdot pr( F_{1} . F_{2} . \neg F_{3} . G_{2})} = pr( H | F_{1} . F_{2} . \neg F_{3} . G_{2})
\end{multline*}
which is what we wanted to show by this example. 
\end{example} 

Theorems \ref{theo:wason} and \ref{theo:exact2list} directly entail the main theorem of this section:
\begin{theorem} \label{theo:7.4} If $\textsc{Exact}(k, U, F) :=$ ``exactly $k$ objects (of the universe $U$) are $F$" and $a \in U$ is an object, and the probability measure $pr$ is exchangeable, weak projectability (PJ) entails:
\begin{align}
pr \big(H \,|\, \textsc{Exact}(k, U, F) .         F_a.         G_a \big)  &>      pr \big( H \,|\, \textsc{Exact}(k, U, F) \big)  \label{eq:77.1} \\
pr \big(H \,|\, \textsc{Exact}(k, U, F) . \neg F_a.         G_a \big)  &\geq pr \big( H \,|\, \textsc{Exact}(k, U, F) \big)  \label{eq:77.2} \\
pr \big(H \,|\, \textsc{Exact}(k, U, F) . \neg F_a. \neg G_a \big) &\leq   pr \big( H \,|\, \textsc{Exact}(k, U, F) \big)  \label{eq:77.3}
\end{align}
and reasoning by analogy (RA) assumption entails:
\begin{align} 
pr \big(H \,|\, \textsc{Exact}(k, U, F) .         F_a.         G_a \big)  &>      pr \big( H \,|\, \textsc{Exact}(k, U, F) \big)  \label{eq:ra.77.1} 
\end{align}
\end{theorem}

The above relations seem to be compatible with intuition. While the total number of objects that satisfy $F$ is known in advance, the consideration of $F_a$ or $\neg F_a$ should not affect our estimation of the frequency of the objects being $F$.
On the other hand, the probability of $G$ can still be affected by observations. Therefore, assuming PJ, consideration of $G_a$ increases the probability of $G$ and consequently decreases the probability of $\Qtf$. 
As a result it seems reasonable that the evidence $G_a$ confirms $H=(\neg \Qtf)_{1:{\usize}}$, and the evidence $\neg G_a$ disconfirms it.

Moreover, an observation $F_a.G_a$ has an extra effect: 
While it is known that only $k$ objects can be counterexamples to $H$ (because in order to be $\Qtf$, one should be $F$), the observation $F_a.G_a$ decreases the number of possible counterexamples by one. This holds even in the case where the chosen measure is such that inductive reasoning is not possible (e.g.\ the uniform measure is used). Consequently, in (\ref{eq:77.1}) inequality is strict, but in (\ref{eq:77.2}) and (\ref{eq:77.3}) it is not. Theorem (\ref{theo:7.4}) implies the following results:
\begin{corollary} \label{cor:a}
If $F$ := raven  and $G$ := black, according to relation (\ref{eq:77.1}) (or \ref{eq:ra.77.1}), PJ (or RA) leads to:\\
$pr$($H |$\,(exactly $k$ objects are ravens).(a specific object is raven and black)) $>$ $pr$($H|$exactly $k$ objects are ravens)
\end{corollary}
\begin{corollary} \label{cor:b}
If $F$ := nonBlack \& $G$ := nonRaven w.r.t.\ relation (\ref{eq:77.1}) (or \ref{eq:ra.77.1}), PJ (or RA) leads to:\\
$pr$($H|$(exactly $k$ objects are not black).(a specific object is nonBlack and nonRaven)) $>$ $pr$($H|$ exactly $k$ objects are not black)
\end{corollary}
\begin{corollary} \label{cor:c}
If $F$ := raven and $G$ := black, w.r.t.\ (\ref{eq:77.2}), PJ leads to:\\
$pr$($H|$(exactly $k$ objects are ravens).(a specific object is nonRaven and black)) $\geq$ $pr$($H|$exactly $k$ individuals are ravens)
\end{corollary}
\begin{corollary} \label{cor:d}
If $F$:=nonBlack \& $G$:=nonRaven w.r.t.\ (\ref{eq:77.2}), PJ leads to:\\
$pr$($H|$(exactly $k$ individuals are not black).(a specific object is black and not raven)) 
$\geq$ $pr$($H|$exactly $k$ individuals are not black)
\end{corollary}
\begin{corollary} \label{cor:e}
If $F$:= raven and $G$ := black, w.r.t.\ (\ref{eq:77.3}), PJ leads to:\\
$pr$($H|$(exactly $k$ individuals are ravens).(a specific object is not raven and not black)) $\leq$ $pr$($H|$exactly $k$ individuals are ravens)
\end{corollary}
\begin{corollary} \label{cor:f}
If $F$:=  nonBlack \& $G$  :=  nonRaven, w.r.t.\ (\ref{eq:77.3}), PJ leads to:\\
$pr$($H|$(exactly $k$ individuals are not black).(a specific object is black and raven)) $\leq$ $pr$($H|$exactly $k$ individuals are not black)
\end{corollary}
\begin{center}
\begin{table}
\begin{tabular}{|p{0.2\textwidth}|p{0.38\textwidth}|p{0.32\textwidth}|}
\cline{2-3} 
\multicolumn{1}{r|}{}& \multicolumn{1}{|l|}{no. of the ravens is known}     & \multicolumn{1}{|l|}{no. of non-blacks is known}\\ 
\hline
observation of a        & \multicolumn{1}{|l|}{PJ: doesn't confirm $H$ (Cor. \ref{cor:e})}     & PJ: confirms $H$ (Cor. \ref{cor:b}) \\               
non-black                  & NC: confirms $H$                                                      & NC: confirms $H$ \\
non-raven                 &RA: n/a                                                                       &RA: confirms $H$ (Cor. \ref{cor:b})\\
\hline
observation of a       & \multicolumn{1}{|l|}{\multirow{2}{*}{evidence refutes $H$}} & \multicolumn{1}{|l|}{\multirow{2}{*}{evidence refutes $H$}}  \\
non-black raven        &                                                                                     &   \\ 
\hline
                                 & \multicolumn{1}{|l|}{PJ: doesn't disconfirm $H$} & \multicolumn{1}{|l|}{PJ: doesn't disconfirm $H$} \\
observation of a       & \multicolumn{1}{|l|}{\hspace{10mm}(Corollary \ref{cor:c})}             & \multicolumn{1}{|l|}{\hspace{10mm}(Corollary \ref{cor:d})} \\
black non-raven	& \multicolumn{1}{|l|}{NC: n/a}                                   & NC: n/a\\                              
			& \multicolumn{1}{|l|}{RA: n/a}                                   & RA: n/a\\                              
\hline
                                & \multicolumn{1}{|l|}{PJ: confirms $H$ (Corollary \ref{cor:a})}                    & \multicolumn{1}{|l|}{PJ: does not confirm $H$}\\
observation of a     & \multicolumn{1}{|l|}{\hspace{10mm}}& \multicolumn{1}{|l|}{\hspace{10mm}(Corollary \ref{cor:f})} \\
black raven	          & NC: confirms $H$					          & NC: confirms $H$ \\
			& RA: confirms $H$	(Corollary \ref{cor:a})		          & RA: n/a \\
\hline
\end{tabular}
\caption{Weak projectability (PJ) vs. Nicod's condition (NC) and reasoning by analogy (RA) in setting 2.}
\label{tab:cors}
\end{table}
\end{center}
Corollaries (\ref{cor:a}) to (\ref{cor:f}) are summarized in (Table \ref{tab:cors}). NC, if assumed to hold in this setting, 
suggests that the observation of a
\emph{non-black non-raven} and the observation of a \emph{black raven}
(i.e.\ entries in the first and fourth rows of the table) should
confirm $H$ which clearly contradicts what PJ suggests, therefore, there is a tension between these two rules.
When background knowledge is neglected, 
intuition goes with PJ in the first column of the table. 
On the other hand, it does not completely match the suggestions of either NC or PC in the second column.
This may indicate that intuition is more inclined to the case
where ``the number of ravens" and not ``the number of non-blacks" is known
in advance. In real life, none of these numbers is known but the
total number of ravens can be estimated much easier than the number
of non-black objects. On the other hand, if we are explicitly
informed of the total number of non-black objects, 
at least in cases similar to the following example, intuition seems to follow
PJ's suggestions in the second column:

Imagine that you
are only concerned about objects which are placed inside a bag (i.e.\
$U:=$ \emph{set of objects inside a bag}). Also imagine that you are
told that only 4 objects are not black. In this case, there are just four possible counterexamples to  $H$. Now
suppose that a \emph{green apple} comes out of the bag. 
Since it is \emph{green}, it is one of those 4 \emph{non-blacks}. Therefore, one
possible counterexample is removed. Meanwhile, the fact that it is
a \emph{non-raven} may increase the probability of
\emph{non-ravenhood} (w.r.t.\ PJ). 
Therefore it is more probable that the 3 remaining \emph{non-blacks} are also \emph{non-raven}. 
Thus, this observation should confirm $H$. 
Now suppose that a \emph{black raven} comes out. 
Its color informs us that it is not among the possible counterexamples but its kind increases the probability of
\emph{ravenhood} which is not in favor of $H$. 
So, this observation cannot confirm $H$. 
Therefore, this example suggests that in Setting 2, given the proper
background knowledge, intuition does not follow NC. 
It follows PJ even if it advises that the observation of a \emph{green apple} confirms that
``all ravens are black" and the observation of a \emph{black raven} does not!

Clearly, we can never ``prove" that a particular measure or a particular proposition is (or is not) ``intuitively plausible", due to the subjective nature of the problem.
The former example presented a particular method of reasoning that relative to a given configuration supports PJ more than NC. This has convinced us that generally in setting 2, PJ is more plausible than NC but as we mentioned, some people might not be convinced. For example, 
one might argue that if we are told that the number of non-black objects is precisely 7 million, we can still believe that black ravens confirm that all ravens are black. 
Such reasoning might be on grounds of some ``hidden" background information such as knowing that ravens are animals and animals of the same kind often have similar colors. 
This particular  background knowledge is not in $\Delta$ (and therefore not in Setting 2) however as it was mentioned in the introduction, this knowledge is convertible to the subjectively chosen a priori probability measure. 
Of course given such knowledge, there will be no surprise if NC holds for $F$:=\emph{raven} and $G$:=\emph{black} but not for $F$:= \emph{non-black} and $G$:= \emph{non-raven}. However, it is up to the readers to judge about what is intuitive for them and what is not. What was formally provable (and is proved formally) is that in setting 2, no probability measure can simultaneously satisfy PJ and NC for a couple of predicates $F$ and $G$. 
\\
{\bf Reasoning by analogy vs. Nicod's condition: }
Table \ref{tab:cors} clearly shows that the only cases where PJ and NC do not contradict is when according to RA, the general hypothesis $H$ should be confirmed. 
In other cases, RA do not impose a restriction; therefore, it never contradicts either PJ or NC.  

A little thought reveals that RA and NC have many commonalities. We go a step further and propose a conjecture that NC may seem intuitively valid since it can easily be conflated with RA as follows:

According to the informal definition of NC: ``The observation of an $F$ that is $G$ confirms that all $F$ are  $G$ (or any $F$ is $G$)."
Although this informal statement seems to be plausible a priori, it is vague and NC is not necessarily its only possible
formalization. To begin with, it should be noticed that in informal language, the scopes of
quantifiers are often ambiguous; For example the informal expression
``for all $b$, the probability of  $\psi_b$ \ldots" can easily be mistaken for ``the probability that for all $b$,  $\psi_b$". 
However the most suitable formalization of the former (i.e.\ $\forall b \quad
pr(\psi_b)$) differs from that of the latter (i.e.\ $pr( \forall b \; \psi_b)$).
On the other hand, the informal ``if" does not exclusively stand for \emph{material implication}; 
in a proper context it can also mean \emph{conditional probability}. 
Putting these together, it can be seen that:
\begin{center}
(Informal NC):\quad {``The observation of an object $a$ which is both $F$ and $G$ confirms that all (or any) object $b$ that is $F$ is also $G$."}
\end{center}
can alternatively be formalized as:
\(
\forall b \quad pr(G_b \, | \, F_b . F_a . G_a) > pr(G_b \,|\, F_b)
\). 
This relation is the definition of RA (see relation \ref{eq:analogy}) -- a rule of induction which is used in many fields (e.g.\ in \emph{case-based reasoning} \cite{Aamodt94}), is directly justified by the principle of the uniformity of nature and does not suffer from the shortcomings of NC such as contradicting PJ or producing counterintuitive conclusions such as PC in the raven paradox.

\section{When the Size of the Universe is Unknown} \label{sect:universe}
In Section \ref{sect:notate}, we defined our probability space using the sample space $\Omega$, the set of all \emph{complete description vectors} (CDVs), all involving $\usize$ objects. Thus, from the beginning, we had to assume that the cardinality of the universe is known. 
In this section, we instead assume that:
\begin{enumerate}\itemsep0pt
\item  The size of the universe (i.e.\ $\usize$) is unknown; however it is known that it is fixed and bounded by some known constants $\alpha$ and $\beta$. E.g.\ assume it is known that the number of the objects of the universe is larger than $\alpha = 10^{10}$ and less than $\beta = 10^{1000}$.   
\item The new evidence $E := F_a.G_a$, does not affect the way the rational agent estimates the size of the universe. E.g., the observation of a black raven does not change the probability distribution over the possible sizes of the universe.
\end{enumerate}
We show that in this setting, our previous conclusions are still valid. Informally speaking, the reason is that all (in)equalities of the previous sections hold for any arbitrary (but fixed) size of the universe, therefore this number does not play a role, and consequently, even if it is unknown, (as long as new evidence does not affect the agent's beliefs about it) all (qualitative) relations should still hold.
 
To justify this claim formally, we need some new notation: 
\\If it is known that the size of universe is $\upsilon$, we let the enumeration of its objects be $U_\upsilon :=\{1, 2, \ldots \upsilon \}$ (created recursively by $U_\upsilon := U_{\upsilon - 1} \sqcup \{ \upsilon \}$).\footnote{
Note that `$\sqcup$' denotes the disjoint union operation. 
}   
Instead of $\Delta$, we write $\Delta_\upsilon$ to indicate that the members of this set describe individuals which belong to the universe $U_\upsilon$. 
Similarly, instead of $\Omega$, we write $\Omega_\upsilon$ to emphasize that the sample space corresponds to a universe of size $\upsilon$ (i.e.\ $U_\upsilon$).\footnote{
While $\Omega_\upsilon$ contains CDVs that exactly describe $\upsilon$ objects, for all distinct $\upsilon'$ and $\upsilon''$, $\Omega_{\upsilon'} \cap \Omega_{\upsilon''}$=$\emptyset$.
}
Similarly, instead of $pr(\cdot)$, we write $\measure_{\upsilon}(\cdot)$ to indicate that by definition, $\measure_{\upsilon}$ (defined on $\Omega_{\upsilon}$) is a measure that provides a probabilistic model for the \emph{rational agent} (who performs induction), only if he/she/it knows the cardinality of the universe is $\upsilon$.

To prevent ambiguity, instead of representing the events by propositions, we directly use subsets of the sample spaces:
The event that corresponds to an arbitrary proposition $\rho$ relative to a sample space $\Omega_{\upsilon}$ is: $\omega_\rho^{\Omega_{\upsilon}} := \{ o \in \Omega_{\upsilon} : o \models \rho \}$.
For example, if $\rho := F_{1}$, then $\omega^{\Omega_1}_{\rho}$ represents the event 
$\{ \boldsymbol{\Qtf_{ 1 }}, \boldsymbol { \Qtt_{ 1 } } \}$,
while $\omega^{\Omega_2}_{\rho}$ stands for the event 
$\{ \boldsymbol{\Qtf_{ 1 } \Qff_{ 2} }$, $\boldsymbol{\Qtf_{ 1 } \Qft_{ 2 } }$, $\boldsymbol{\Qtf_{ 1 } \Qtf_{ 2 } }$, $\ldots$, $\boldsymbol{\Qtt_{ 1 } \Qtt_{ 2 } } \}$.
The \emph{Complete Description Vectors} (CDVs) are here bold-faced and conjunction symbols ``." are dropped to emphasize that they are not ordinary propositions. 
For example $\boldsymbol{ \Qtt_{ 1 } \Qtt_{ 2} }$ is a CDV that not only entails the ordinary proposition $\Qtt_{ 1 } . \Qtt_{ 2 }$, but also indicates that the universe is $U_2 = \{1$, $2\}$. The reason is that by definition, each CDV describes all objects of the universe (see Section \ref{sect:notate}).

Let $\alpha$ and $\beta$ be some known lower and upper bound for the size of the universe.\footnote{
$\alpha$ is at least equal to the number objects mentioned by background knowledge or evidence. $\beta$ can be arbitrarily large but for simplicity, we assume that it is finite. To see what would be needed if we wanted to allow $\beta \rightarrow \infty$, refer to Footnote  \ref{foot:infinite}.
} 
We define a new sample space $\bigss$ as a set that contains all members of all sample spaces that correspond to universes with sizes at least equal to $\alpha$ and at most equal to $\beta$:
\begin{equation} \label{eq:big-omega}
\forall \alpha, \beta \in \mathbb{N} \, \text{ s.t.\ } \alpha \leq \beta \qquad \bigss := \bigsqcup_{\upsilon=\alpha}^{\beta} \Omega_\upsilon
\end{equation}
Relative to the sample space $\bigss$ and for all $\upsilon \in [\alpha, \beta]$, the subset $\Omega_\upsilon$  
represents the event that ``the size of the universe is $\upsilon$". We let:
\begin{equation*}
\omega_\rho := \{ o \in \bigss : o \models \rho \} = \bigsqcup_{\upsilon = \alpha}^{\beta} \{ o \in \Omega_\upsilon : o \models \rho \} =
 \bigsqcup_{\upsilon = \alpha}^{\beta} \omega_\rho^{\Omega_\upsilon}  =  \bigsqcup_{\upsilon = \alpha}^{\beta} (\omega_\rho^{\Omega_\upsilon} \cap \Omega_\upsilon)
\end{equation*}
$\omega_\rho$
corresponds to the event that regardless of the size of the universe, proposition $\rho$ holds.\footnote{
More specifically, $\omega_{\rho}$ represents: ``$\usize$=$\alpha$ and  $\rho$ w.r.t.\ $U_\alpha$" or ``$\usize$=$\alpha$+$1$ and  $\rho$ w.r.t.\ $U_{\alpha + 1}$" or \ldots or ``$\usize$=$\beta$ and  $\rho$ w.r.t.\ $U_\beta$".
} We refer to $\omega_\rho$ as the \emph{generalized correspondent of $\rho$}. 
While 
$\omega^{\Omega_{\upsilon}}_\rho \subseteq \Omega_\upsilon$ 
and for all 
$\upsilon' \neq \upsilon''$, $\Omega_{\upsilon'} \cap \Omega_{\upsilon''} = \emptyset$, the above relation entails: 
\begin{equation} \label{eq:omega-psi}
\omega_\rho^{\Omega_\upsilon} = \omega_{\rho} \cap \Omega_{\upsilon}
\end{equation}
Likewise, the event that represents ``exactly $k$ objects of a universe of unknown size (but bounded by $\alpha$ and $\beta$) are $F$" is defined as follows: 
\begin{equation*}
\omega_{\textsc{Exact}(k, F)} :=  \bigsqcup_{\upsilon = \alpha}^{\beta} \omega_{\textsc{Exact}(k, U_\upsilon, F)}^{\Omega_\upsilon} 
=  \bigsqcup_{\upsilon = \alpha}^{\beta} ( \omega_{\textsc{Exact}(k, U_\upsilon, F)}^{\Omega_\upsilon} \cap \Omega_\upsilon )
\end{equation*}
We refer to $\omega_{\textsc{Exact}(k, F)}$ as the \emph{generalized correspondent of} $\textsc{Exact}(k, U_\alpha, F)$ to $\textsc{Exact}(k, U_\beta, F)$.

Over the sample space $\bigss$, we let $\bigmeasure$ denote the probability measure that explains the \emph{rational agent}'s a priori degrees of beliefs in the events. As mentioned in footnote~(\ref{foot:peter}), the existence of such a measure is deduced from the rationality axioms\footnote{
\label{foot:infinite} 
For the case where $\bigss$ is finite, Savage’s axioms are sufficient. If we wanted to study the case where it is infinite (i.e.\ no upper bound exists:  $\beta \rightarrow \infty$), we needed to add the \emph{monotone continuity assumption} \cite{arrow70} to the rationality axioms to guarantee countable additivity.
}
\cite{savage54}. 
Note that:
\begin{equation} \label{eq:mixed-restrict}
\forall \omega \subseteq \bigss, \forall \upsilon \in \mathbb{N} \text{ s.t.\ } \alpha \leq \upsilon \leq \beta \qquad 
\bigmeasure (\omega \,|\, \Omega_\upsilon ) = \measure_\upsilon (\omega)
\end{equation}
The reason is that by definition, both sides of the above equation represent the probability measure chosen by the \emph{rational agent}, when the size of the universe is known to be $i$.\footnote{
Evidently, if $\bigmeasure(\Omega_\upsilon) = 0$, both  $\bigmeasure (\omega \,|\, \Omega_\upsilon )$ and $\measure_\upsilon(\omega)$ are undefined.
}

Let $E := F_a . G_a$ where $a \in U_\alpha$ (consequently for all $\upsilon \geq \alpha$, $a \in U_\upsilon$), and $\omega_D$ be either the generalized correspondent of $\textsc{Exact}(k, U_\alpha, F)$ to $\textsc{Exact}(k, U_\beta, F)$ or the generalized correspondent of any $D \in \Delta_\alpha$. 
Roughly speaking, this means that $\omega_D$ is the generalized correspondent of any piece of background knowledge  which is discussed in the previous sections.
Based on these notations, relation (\ref{eq:mixed-assumption3}) represents a formal version of Assumption (2) (i.e.\ the assumption that the estimated size of the universe is not affected by evidence $E$):
\begin{equation}\label{eq:mixed-assumption3}
\text{(Assumption)} \;
\forall \upsilon \in [\alpha, \beta] \text{ s.t. } \bigmeasure (\Omega_\upsilon) > 0 \quad \bigmeasure(\Omega_\upsilon \,|\, \omega_E \cap \omega_D) 
= \bigmeasure(\Omega_\upsilon \,|\, \omega_D)
\end{equation}

Having this setting, it is easy to show that all (in)equalities of the previous sections hold for the case where the exact number of objects in the universe is unknown. As an example, consider the relation (\ref{eq:typical}). This is a typical inequality that informally speaking states that ``for any possible universe-size $\upsilon$, if it is known that the size of the universe is equal to $\upsilon$, relative to background knowledge $D$, the evidence $E$ confirms the hypothesis $H$." (Note that by the notation that we were using in the previous sections, this equation would be represented by $pr(H|E.D) > pr(H|D)$ where the size of the universe was not mentioned explicitly.)  
\begin{equation}\label{eq:typical}
\forall \upsilon \in \mathbb{N}  \text{ s.t. } \alpha \leq \upsilon \leq \beta, \, \bigmeasure(\Omega_\upsilon)>0 \qquad \measure_\upsilon (\omega^{\Omega_\upsilon}_H| \, \omega^{\Omega_\upsilon}_E \cap \omega^{\Omega_\upsilon}_D) > \measure_\upsilon (\omega^{\Omega_\upsilon}_H | \, \omega^{\Omega_\upsilon}_D) 
\end{equation}
Proposition \ref{prop:example} proves that the mentioned typical inequality entails that: 
$\bigmeasure(\omega_H | \, \omega_E \cap \omega_D) > \bigmeasure(\omega_H| \, \omega_D)$, which informally speaking asserts that relative to a universe of an unknown size (that is bounded by $\alpha$ and $\beta$), in the presence of background knowledge $D$, the evidence $E$ confirms $H$. 
\begin{proposition} \label{prop:example}
Let the size of the universe be unknown but known to be bounded by $\alpha$ and $\beta$.  Let $\bigmeasure$ be a probability measure that corresponds to the degrees of beliefs of a rational agent in the events of the sample space $\bigss$, defined by relation (\ref{eq:big-omega}). If $\bigmeasure$ complies with relation (\ref{eq:mixed-assumption3}), then relation (\ref{eq:typical}) entails:
\begin{equation*}
\bigmeasure(\omega_H | \, \omega_E \cap \omega_D) > \bigmeasure(\omega_H| \, \omega_D)
\end{equation*} 
\end{proposition}
\begin{proof}
Let $S:=\{s \in \mathbb{N} \;:\; \alpha \leq s \leq \beta, \; \bigmeasure(\Omega_s)>0 \}$ be the set of all possible sizes of the universe.\\[2mm]
$\forall \upsilon \in S \quad \measure_\upsilon(\omega^{\Omega_\upsilon}_H \,|\, \omega^{\Omega_\upsilon}_E \cap \omega^{\Omega_\upsilon}_D) > \measure_\upsilon(\omega^{\Omega_\upsilon}_H \,|\, \omega^{\Omega_\upsilon}_D)$, by  (\ref{eq:typical}) and definition of $S$ \\[2mm]
$\Rightarrow \forall \upsilon \in S \quad \bigmeasure(\omega^{\Omega_\upsilon}_H \,|\, \omega^{\Omega_\upsilon}_E \cap \omega^{\Omega_\upsilon}_D \cap \Omega_\upsilon) > 
\bigmeasure(\omega^{\Omega_\upsilon}_H \,|\, \omega^{\Omega_\upsilon}_D \cap \Omega_\upsilon) 
, \; \text{by relation (\ref{eq:mixed-restrict})} $\\[2mm]
$\Rightarrow \forall \upsilon \in S \quad \bigmeasure\big( \omega_H \cap \Omega_\upsilon \,|\, 
(\omega_E \cap \Omega_\upsilon) \cap (\omega_D \cap \Omega_\upsilon) \cap \Omega_\upsilon \big) > 
\bigmeasure \big ( \omega_H \cap \Omega_\upsilon \,|\, (\omega_D \cap \Omega_\upsilon) \cap \Omega_\upsilon \big)$, by relation (\ref{eq:omega-psi}) \\ [2mm]
$\Rightarrow \forall \upsilon \in S \quad \bigmeasure( \omega_H \,|\, \omega_E \cap \omega_D \cap \Omega_\upsilon) > 
\bigmeasure ( \omega_H   \,|\,  \omega_D \cap \Omega_\upsilon), \; \text{by simplification} $\\[2mm]
$\Rightarrow \forall \upsilon \in S \quad 
\bigmeasure(\Omega_\upsilon \,|\, \omega_E \cap \omega_D) \cdot
\bigmeasure( \omega_H \,|\, \omega_E \cap \omega_D \cap \Omega_\upsilon) >
\bigmeasure(\Omega_\upsilon \,|\, \omega_D) \cdot
\bigmeasure ( \omega_H   \,|\,  \omega_D \cap \Omega_\upsilon ), \;$ by multiplying the r.h.s.\ and l.h.s. by the r.h.s.\ and l.h.s.\ of relation (\ref{eq:mixed-assumption3}) \\[-2mm]
\begin{equation*}
\Rightarrow \sum_{\upsilon \in S}  
\bigmeasure(\Omega_\upsilon \,|\, \omega_E \cap \omega_D) \cdot
\bigmeasure( \omega_H \,|\, \omega_E \cap \omega_D \cap \Omega_\upsilon ) >
\sum_{\upsilon \in S}
\bigmeasure(\Omega_\upsilon \,|\, \omega_D) \cdot 
\bigmeasure ( \omega_H   \,|\,  \omega_D \cap \Omega_\upsilon) \quad
\end{equation*}
The above inequality is equivalent to $\bigmeasure(\omega_H |\, \omega_E \cap \omega_D) > \bigmeasure(\omega_H |\, \omega_D)$ (where $\Omega_\upsilon$ is marginalized out), which is the intended result.
\end{proof}

\section{Conclusion} \label{sect:conclude}

We argued that from a Bayesian perspective, (a) objective background information (from previous observations) and subjective prior information (i.e.\ prior degrees of beliefs) can produce the same effects and in this sense, are convertible; therefore, in this context, ``the state of perfect ignorance" should be interpreted as ``the state of possessing no objective information and no subjective biased beliefs". On the other hand, if we assume that the uniform probability measure corresponds to unbiased subjective degrees of beliefs (as it is often assumed), then we should conclude that:
(b) with unbiased subjective beliefs, inductive reasoning is impossible.
Therefore, based on (a) and (b), induction in a condition of perfect ignorance is impossible.  

In addition, by examples we have shown that relative to unrestricted objective/subjective prior information, common rules of induction do not always hold. We concluded that rules of induction should be considered plausible, if they hold relative to a large class of plausible (objective) background knowledge (i.e.\ knowledge similar to our actual background knowledge) and plausible probability measures (i.e.\ measures with reasonable characteristics such as complying with more intuitive rules of induction).  
Subsequently, we scrutinized the plausibility of NC by fixing the background knowledge and studying the characteristics of measures that do or do not comply with it.

In the first setting, the background knowledge is
composed of \emph{complete descriptions} of several objects. 
It is shown that in this setting, validity of NC is implied by the
answer to a simpler question that does not seem to have a general
intuitive answer i.e.\ for distinct objects $a$ and $b$, whether $E=F_a.G_a$ confirms $F_b. \neg G_b$ or not.
While due to the chosen probability measure and characteristics of
$F$ and $G$, the latter condition does not hold in general, we concluded
that in this setting NC is not significantly more reasonable than $\neg$NC.

In the second setting, the number of objects satisfying a particular predicate is known by background knowledge. It is shown that in this
case, NC may contradict PJ while seemingly, intuition follows the
latter. 
In summary:
\begin{enumerate}\itemsep0pt
\item There are \emph{reasonable} (i.e.\ not implausible a priori) probability measures for which NC does not hold;
\item There are \emph{reasonable} probability measures for which weak projectability, i.e.\ one of the simplest forms of inductive inference, opposes NC;
\item In the case of contradiction, intuition ``seems to" follow projectability rather than NC;
\end{enumerate}
Hence, we conclude that we have gathered some evidence against the assumption that NC is a generally reliable rule of induction. 
If NC is not considered plausible then the raven paradox is also
dispelled since  even if NC holds for two predicates $F$ and $G$,
it does not mean that it should hold for $\neg G$ and $\neg F$.
This asymmetry may be due to possible asymmetry in background knowledge and/or priors.

Of the three mentioned points, point (1) had already been demonstrated \cite{maher04}.   
The distinction is that we have dealt with NC without relying on any particular a priori measure.
As mentioned throughout the paper, in the Bayesian framework one's choice of prior
probability distribution is subjective. Nevertheless, 
Bayesian agents' beliefs at any time are (heavily) 
dependent on their prior distributions. Therefore,  
different subjective choices may lead to contradictory results.
In the case of variations of Carnap's measure, this freedom of 
choice also resurfaces in the form of one or several arbitrary parameters.
\cite{maher99} and \cite{maher04} are two variations of
Carnap's measure but for the former, NC always holds and for the latter, for 
some parameter configurations, it does not. 
Although one might claim that the existence of at least one measure that does not comply with NC is sufficient to discredit NC, one
should note that any rule of induction can be violated by some measures and remains 
valid for some others. In fact one might contrarily claim that 
contradiction of a measure with an ``intuitive" rule of induction  
should discredit that particular measure rather than the rule.
Therefore, our sufficient condition for ($\neg$NC) that covers a class of measures rather than a specific one, should be considered a more general and interesting relation.

The main contribution of our work is the conjunction of points (2) and (3). We believe that compared to the raven paradox's PC and other counterintuitive consequences of NC, the conflict between NC and PJ is more important. 
To our knowledge, other counterintuitive consequences (such as \emph{The red herring} \cite{good67} and \emph{Good's baby} \cite{good68} problems) are either related to arguably implausible background configurations, or as is the case with PC, some claim that they are not counterintuitive. PJ is a very simple form of inductive inference. 
It is more directly justified by our intuitive notion of inductive reasoning than NC, and it's plausibility cannot be challenged as easily. 

We also proposed a conjecture that NC seems to be plausible because it can be mistaken for reasoning by analogy (RA) which has a closely related informal representation.
We proved that in the case where the exact number of objects satisfying one property is known, RA is compatible with both PJ and NC. It is a conservative condition that intuitively seems  plausible and does not suffer from the shortcomings of NC (such as contradicting PJ or producing counterintuitive conclusions such as PC) and is directly justified by the principle of the uniformity of nature.

\section{Proof of Theorems}
\label{sect:proofs}
\subsection{Proof of Theorems \ref{theo:a2-a3-nc} and \ref{theo:a1-nc}} \label{sect:setting1_proofs}
\begin{lemma} \label{theo:xis}
In relation (\ref{eq:nc}), if $D \in \delta$, NC is equivalent to $\Xi_1 \cdot \Xi_2 > 1$ where:
$U \backslash \inds_D := \{{b_1}, \ldots {b_n}, a\}$ 
is ``the set of objects not described by $D$",
\\$\Xi_1 := \frac{pr(\impl_{ {b_1} : {b_n}} | \Qtt_{ a } . D)}{pr(\impl_{ {b_1} : {b_n} } | D)}$ and $\Xi_2 := \frac{1}{pr(\impl_a | \impl_{ {b_1} : {b_n} } . D)}$ .
\end{lemma}
\begin{proof}
Objects are either completely described by $D \in \delta$ or are not described at all. Therefore: 
\begin{align}
\label{eq:xis1}
pr(\impl_{ 1 : {\usize} } | D) &= pr(\impl_{ {b_1} : {b_n} }. \impl_a | D) \\
\label{eq:xis2}
pr(\impl_{ 1 : {\usize} } | \Qtt_{ a } . D) &= pr(\impl_{  {b_1} : {b_n} } | \Qtt_{ a } . D)
\end{align}
Combination of (\ref{eq:xis1}), (\ref{eq:xis2}) and (\ref{eq:nc}) proves the lemma.\footnote
{Note that $\Xi_1$ indicates the effect of
the observation $F_a.G_a$, on the probability that 
unobserved individuals satisfy $\impl$ and $\Xi_2$ 
corresponds to the effect of elimination of
the possibility that the observed object is a
counterexample to the generalization. 
}
\end{proof}
\begin{proof}[{\bf Proof of Theorem \ref{theo:a2-a3-nc}}]
Due to $pr(\Qtf_{ b } | \cdot ) = 1 - pr(\impl_{b} | \cdot)$, relation (\ref{eq:bq1}) is equal to:
\begin{equation} \label{eq:bq2}
\forall B \in \Delta, \forall a, b \notin \inds_B \qquad pr(\impl_{b} | \Qtt_{ a } . B) \geq pr(\impl_{b} | B)
\end{equation}
Let arbitrary $D \in \delta$ and $\inds_D = U \backslash \{{b_1}, \ldots, {b_n}, a\}$ (i.e.\ $b_1$ to $b_n$ are the objects not mentioned by background knowledge or evidence). 
According to the definitions of $\delta$ and $\Delta$, for all $i < n$:
$\impl_{{b_1}} \!\! \ldots \impl_{{b_i}} . D \in \Delta$ (because it is consistent) and $a, {b_{i+1}} \notin \inds_{ \impl_{{b_1}} \!\!\! \ldots \impl_{{b_i}} . D}$. Therefore, from (\ref{eq:bq2}) it follows that:
\begin{multline} \label{eq:prod-mult}
\forall i<n \quad pr(\impl_{{b_{i+1}}} | \impl_{{b_1} : {b_i}} .\Qtt_{ a } . D) \geq pr(\impl_{{b_{i+1}}} | \impl_{{b_1} : {b_i}} . D) \\ \Longrightarrow
\prod_{i=0}^{n-1}{ pr(\impl_{{b_{i+1}}} | \impl_{{b_1} : {b_i}} .\Qtt_{ a } . D)} \geq 
\prod_{i=0}^{n-1}{pr(\impl_{{b_{i+1}}} | \impl_{{b_1} : {b_i}} . D)} 
\end{multline}
Due to the chain rule, the r.h.s.\ of the above equation is equal to $pr(\impl_{{b_1} : {b_n}} | D)$:
\begin{multline*} 
\prod_{i=0}^{n-1}{\! pr(\impl_{{b_{i+1}}} \!| \impl_{{b_1} : {b_i}} . D)} =\\
pr(\impl_{{b_1}} \! | D) \cdot pr(\impl_{b_{2}} \! | \impl_{b_{1}} \! . D) \ldots 
pr(\impl_{{b_n}} \! | \impl_{{b_1}} \!\! \ldots \impl_{{b_{n-1}}} \! . D) = pr(\impl_{{b_1} : {b_n}} | D) 
\end{multline*}
Similarly, the l.h.s.\ of (\ref{eq:prod-mult}) is equal to $pr(\impl_{{b_1} : {b_n}} | \Qtt_{ a } . D)$. Therefore:
\begin{equation*}
pr(\impl_{{b_1} : {b_n}} | \Qtt_{ a } . D) \geq pr(\impl_{{b_1} : {b_n}} | D) 
\end{equation*}
This entails $\Xi_1 \geq 1$ and since $\Xi_2 > 1$, according to
Lemma \ref{theo:xis}, NC holds.
\end{proof}
\begin{proof}[{\bf Proof of Theorem \ref{theo:a1-nc}}]
Similar to the method used in the proof of Theorem \ref{theo:a2-a3-nc} 
and by using inequality (\ref{eq:not-nc1}) instead of inequality (\ref{eq:bq1}) it can be proved that: 
\begin{equation*}
\forall i<n \quad
pr(\impl_{{b_{i+1}}} | \, \impl_{{b_1} : {b_i}} . \, \Qtt_{ a } . \, D) 
<
pr(\impl_{{b_{i+1}}} | \, \impl_{{b_1} : {b_i}} . \, D)
\end{equation*}
Consequently:
\begin{equation}
\label{xi1-less}
\Xi_1 = \frac{\prod_{i=1}^{n-1}{ pr(\impl_{{b_{i+1}}} | \impl_{{b_1} : {b_i}} .\Qtt_{ a } . D)}}
{\prod_{i=1}^{n-1}{pr(\impl_{{b_{i+1}}} | \impl_{{b_1} : {b_i}} . D)}} 
\cdot \frac{pr(\impl_{{b_1}} | \Qtt_{ a } . D)}{pr(\impl_{{b_1}} | D)}
<
\frac{pr(\impl_{{b_1}} | \Qtt_{ a } . D)}{pr(\impl_{{b_1}} | D)} 
\end{equation}
For conciseness, let $p := pr(\impl_{{b_1}} | \, \Qtt_{ a } . D)$ and $q := pr(\impl_{ {b_1} } | D)$. 
\begin{align} \label{eq:1qp}
&pr(\neg \impl_a | \impl_{ {b_1} : \, {b_n}} . D) < (1 - p) - (1 - q),  &&\text{since $\impl = \neg \Qtf$ and by inequality (\ref{eq:not-nc2})} \notag\\ 
&\Longrightarrow\, pr(\impl_a | \impl_{ {b_1} : \, {b_n} } . D) > 1 -q + p  
\end{align}
Finally,
\begin{align*}
&1 - p > 1 - q \,\,\Longrightarrow\,\, p<q,													&&\text{since $D \in \Delta$ and by (\ref{eq:not-nc1})}\\
&\Longrightarrow \, p \cdot (1 - q) < q \cdot (1-q) \,\,\Longrightarrow\,\, p < q \cdot (1 - q + p), 					&&\text{since $q < 1$} \\
&\Longrightarrow \, \frac{p}{q} < 1- q + p 													&&\text{since $q > 0$} \\
&\Longrightarrow \Xi_1 < pr(\impl_{a} | \, \impl_{ {b_1} \! : {b_n} } . D),   									&&\text{by (\ref{xi1-less}) and (\ref{eq:1qp})}\\ 
&\Longrightarrow \Xi_1 \cdot \Xi_2 < 1 \Longrightarrow \text{NC does not hold.} 							&&\text{by Lemma \ref{theo:xis}}
\end{align*}
\end{proof}
 \subsection {Proof of Theorem \ref{theo:wason}}
\begin{lemma}\label{theo:group-proj} 
Under PJ, for any set of objects $\{1,2,\ldots$, $n\}\subseteq U$, $a \in U$ and any $D \in \Delta$ that does not determine the value of $\psi_a$:
\begin{align}
&\text{(Group PJ)} \qquad &&pr(\psi_{1: n} | \, \psi_a.D) \geq pr( \psi_{1: n}| \, D) 
\label{eq:group_pj}\\
&\text{(Negative Group PJ)} \qquad &&pr(\psi_{1: n} | \, \neg \psi_a.D) \leq pr( \psi_{1: n}| \, D)
\label{eq:neg_group_pj}
\end{align}
\end{lemma}
\begin{proof}
The proof is based on mathematical induction.\\
1. Proof of relation (\ref{eq:group_pj}):\\
(I)\quad  For $n=1$, relation (\ref{eq:group_pj}) is equivalent to relation (\ref{eq:weak-proj}) and therefore valid.\\
(II)\quad Assume for $n=k$ relation (\ref{eq:group_pj}) holds. 
We prove that for $n=k+1$, it holds as well:
\begin{align*}
pr&(\psi_{1: {k+1}} | \, \psi_a . D) = pr( \psi_{1: k} . \psi_{{k+1}} | \, \psi_a . D) &&\quad\\
&= pr(\psi_{{k+1}} | \psi_a.D) \cdot pr(\psi_{1: k} | \, \psi_a . \psi_{{k+1}} . D) &&\quad\\
&\geq pr(\psi_{{k+1}} | \, D) \cdot pr(\psi_{1: k} | \, \psi_a . \psi_{{k+1}} . D), &&\!\!\!\!\!\!\!\!\!\! \text{by inequality (\ref{eq:weak-proj})}\\
&\geq pr(\psi_{{k+1}} | \, D) \cdot pr(\psi_{1: k} | \, \psi_{{k+1}} . D), &&\!\!\!\!\!\!\!\!\!\! \text{(\ref{eq:weak-proj}) applied to $(\psi_{{k+1}}.D) \in \Delta$}\\
&= pr( \psi_{1: {k+1}} |D) &&\quad
\end{align*}
By (I) \& (II), mathematical induction implies (\ref{eq:group_pj}) for all $n$.\\ 
2. Proof of relation (\ref{eq:neg_group_pj}):\\ 
Under PJ: $\forall b \quad pr(\neg \psi_b | \, \neg\psi_a . D) \geq pr(\neg \psi_b | \, D)$, therefore:
\begin{equation}\label{eq:neg_pj}
 1 - pr(\psi_b | \, \neg \psi_a . D) \geq 1 - pr(\psi_b | \, D) 
\; \Longrightarrow\;  pr(\psi_b | \, \neg\psi_a . D) \leq pr(\psi_b | \, D)
\end{equation}
Similar to the previous case and by using (\ref{eq:neg_pj}) instead of (\ref{eq:weak-proj}), inequality (\ref{eq:neg_group_pj}) can be proved easily.
\end{proof}
\begin{lemma}\label{lem:2} For all $1\leq k \leq \usize$, $D := F_{1:k}.(\neg F)_{{k+1}:{\usize}}$ and $E$ being an arbitrary proposition: 
$ pr(H | E \, . \, D) = pr(G_{1:k} | E \, . \, D)$.
\end{lemma}
\begin{proof}
$
\text{l.h.s.}
=pr \big( \impl_{1:{k}}.\impl_{{k+1}:{\usize}} | E \,.\, F_{1:k}.(\neg F)_{{k+1}:{\usize}} \big)
\\=pr \big(\impl_{1:k} . F_{1:k} . \impl_{{k+1}:{\usize}} . (\neg F)_{{k+1}:{\usize}} |
E \,.\, F_{1:k}.(\neg F)_{{k+1}:{\usize}} \big) \\
= pr \big( (\impl.F)_{1:k}.(\impl.\neg F)_{{k+1}:{\usize}} | E \,.\, F_{1:k} . (\neg F)_{{k+1}:{\usize}} \big)
\\=pr \big( (F.G)_{1:k}.(\neg F)_{{k+1}:{\usize}} | E \,.\, F_{1:k} . (\neg F)_{{k+1}:{\usize}} \big) 
$ (since $\impl_b.F_b$ $\equiv$ $F_b.G_b$ and $\impl_b$.$\neg F_b$ $\equiv$ $\neg F_b$)
$
\\=pr \big( F_{1:k}.G_{1:k}.(\neg F)_{{k+1}:{\usize}} | E \, . \, F_{1:k} . (\neg F)_{{k+1}:{\usize}} \big)\\
=pr \big( G_{1:k} | E \, . \, F_{1:k}.(\neg F)_{{k+1}:{\usize}} \big) =
$ r.h.s.
\end{proof}
\begin{proof}[{\bf Proof of Theorem \ref{theo:wason}}] In the following relations, $D: = F_{1:k}.(\neg F)_{{k+1}:{\usize}}$.\\
(I) \quad Proof of relation (\ref{eq:7.1}), i.e.\ $pr(H \,| \, G_{k}\,.\, D) > pr(H| \, D)$ by PJ:
\begin{align*}
&pr(H | \, G_{k} \,.\, D) =pr(G_{1:k} | \, G_{k} \,.\, D),                            	 	&&\!\!\!  \text{by Lemma \ref{lem:2}}\\
&=pr(G_{1:{k-1}} | \, G_{k} \,.\, D) \geq pr(G_{1:{k-1}} | \, D),		   	&&\!\!\!  \text{by Lemma \ref{theo:group-proj}}\\
&>pr(G_{1:{k}}| \, D) ,                                                                 			&& \!\!\! \text{adding  $G_{k}$ \& Cournot's pp. (note: $D \nvdash G_{k}$)}\\
&= pr(H| \, D),                                                   					&& \!\!\! \text{by Lemma \ref{lem:2}.}
\end{align*}
(II)\quad Proof of relation (\ref{eq:7.2}), i.e.\ $pr(H \,| \, G_{{\usize}} \,.\, D) \geq pr(H| \, D)$ by PJ:
\begin{align*}
&pr(H | \, G_{{\usize}} \,.\, D) = pr(G_{1:k} | \, G_{{\usize}} \,.\, D),             && \!\! \text{ by Lemma \ref{lem:2}}\\
&\geq pr(G_{1:k} | \, D),                                                                            && \!\! \text{ by relation (\ref{eq:group_pj}) (Group PJ)}\\
&=pr(H| \, D),                                                                                                                && \!\! \text{ by Lemma \ref{lem:2}.}
\end{align*}
(III)\quad Proof of relation (\ref{eq:7.3}) i.e.\ $pr(H \,| \, \neg G_{{\usize}} \,.\, D) \leq pr(H| \, D)$ by PJ:
\begin{align*}
&pr(H | \neg G_{{\usize}} \,.\, D) = pr(G_{1:k} | \neg G_{{\usize}} \,.\, D), && \text{ by Lemma \ref{lem:2}}\\
&\leq pr(G_{1:k} | \, D),                                                                           && \text{ by relation\,(\ref{eq:neg_group_pj})\,(Neg.\,Group\,PJ)}\\
&=pr(H| \, D),                                                                                                                 && \text{ by Lemma \ref{lem:2}.}
\end{align*}
(IV)\quad Proof of relation (\ref{eq:ra.7.1}) , i.e.\ $pr(H \,| \, G_{k}\,.\, D) > pr(H| \, D)$ by RA:
\begin{align*}
&pr(H | \, G_{k} \,.\, D) =pr(G_{1:k-1} | \, G_{k} \,.\, D),                            	 	&&\!\!\!  \text{by Lemma \ref{lem:2}}\\
&=\prod_{i=1}^{k-1}{\! pr(G_{i}| G_{1 : i-1} . G_{k}.D)},		   			&&\!\!\!  \text{by the chain rule}\\
&=\prod_{i=1}^{k-1}{\! pr(G_{i}| F_i . F_k . G_k . G_{1 : i-1} . D)},	   			&&\!\!\!  \text{since } D \vdash F_i . F_k\\
&>\prod_{i=1}^{k-1}{\! pr(G_{i}| F_i . G_{1 : i-1} . D)},	   			&&\!\!\!  \text{by RA (and since $G_{1 : i-1} . D \in \Delta$) }\\
&=\prod_{i=1}^{k-1}{\! pr(G_{i}| G_{1 : i-1} . D)},	   			&&\!\!\!  \text{since for $i \in \{1, \ldots, k\}: D \vdash F_i$ }\\
&=pr(G_{1:{k-1}}| \, D) ,                                                        	&& \!\!\! \text{by the chain rule }\\
&> pr(G_{1:k}| \, D) ,                                                        	&& \!\!\! \text{by adding  $G_{k}$ \& Cournot's pp. (note: $D \nvdash G_{k}$)}\\
&= pr(H| \, D),                                                   					&& \!\!\! \text{by Lemma \ref{lem:2}.}
\end{align*}
\end{proof}
\subsection{Proof of Theorem \ref{theo:exact2list}} \label{subsect:premute_proof}
\begin{definition} 
Having a set $C \subseteq U$ and $\permute$ being a permutation (of $U$), $C^\permute$ is defined as: 
$C^\permute := \{\permute(b) : b \in C\}$. It is obvious that $U^\permute = U$.
\end{definition}
The following two relations directly follows from the assumption that $\permute$ is a bijection:
\begin{align}
&\forall C \subseteq U &&|C| = |C^\permute| \label{eq:biject1}\\
&\forall C, C' \subseteq U &&C \neq C' \; \Longrightarrow \; C^\permute \neq C'^\permute \label{eq:biject2}
\end{align}  
\begin{lemma}\label {lem:h_permute}
$H = H^{\permute}$ where $\permute$ is an arbitrary permutation (of $U$).
\end{lemma}
\begin{proof}
$\displaystyle H^\permute = \bigwedge_{b \in U} \!\! \impl_{\permute(b)} 
                = \bigwedge_{\permute(b) \in U^\permute} \!\!\!\! \impl_{\permute(b)}
	     = \bigwedge_{b' \in U^\permute} \!\! \impl_{b'}
                = \bigwedge_{b' \in U} \!\! \impl_{b'}
 	     = H$

The first equality holds by definition. 
The second equality holds because:\\ $b \in U \iff \permute(b) \in U^\permute$. 
In the r.h.s.\ of the third equality, $\permute(b)$ is renamed to  $b'$.
The fourth equality holds because $U = U^\permute$ (note that $\permute$ is a permutation in $U$). The last equality is the definition of $H$.
\end{proof}

\begin{lemma} \label{lem:bigC_unique} 
For all $1 \le k \le \usize$, all permutations $\permute$ and $\mathfrak{C}_{U,k}$ being defined as the set of all subsets of $U$ which have size $k$: 
\begin{equation} \label{eq:gandak}
\mathfrak{C}_{U,k}^\permute := \{ C^\permute : C \in  \mathfrak{C}_{U,k} \} =  \mathfrak{C}_{U,k}
\end{equation}
\end{lemma}
\begin{proof}
Equality (\ref{eq:biject1}) implies, $\forall C^\permute \in \mathfrak{C}_{U,k}^\permute \quad |C^\permute| = |C| = k$, which means: $C^\permute$ is a k-combination from $U$; So from the  definition of $\mathfrak{C}_{U,k}$ it follows that $C^\permute \in \mathfrak{C}_{U,k}$. Therefore, $\mathfrak{C}_{U,k}^\permute  \subseteq \mathfrak{C}_{U,k}$.
Conversely, according to relation (\ref{eq:biject2}) and the definition of $\mathfrak{C}_{U,k}^\permute$ in (\ref{eq:gandak}), 
there is a one-to-one relation between the members of $\mathfrak{C}_{U,k}$ and $\mathfrak{C}_{U,k}^\permute$. This entails: $|\mathfrak{C}_{U,k}^\permute| = |\mathfrak{C}_{U,k}|$. Thus: $\mathfrak{C}_{U,k}^\permute = \mathfrak{C}_{U,k}$. 
\end{proof}
\begin{lemma} \label {lem:exact_permute}
For all permutations $\permute$ (of the set $U$), and for all integers $k \in [1, \usize]$: $\textsc{Exact}(k , U, F) = \textsc{Exact}^\permute (k , U, F)$
\end{lemma}
\begin{proof} By definition, $\displaystyle \textsc{Exact}^{\permute}(k, U, F) := \bigvee_{C \in \mathfrak{C}_{U,k}}  (  \bigwedge_{a \in C} F_{\permute(a)} \;. \bigwedge_{b \not\in C} \neg F_{\permute(b)}  )$
\begin{align*}
& = \bigvee_{C \in \mathfrak{C}_{U,k}}  (  \bigwedge_{\permute(a) \in C^\permute} \!\!\!\!\! F_{\permute(a)} \;. \!\!\! \bigwedge_{\permute(b) \not\in C^\permute} \!\!\!\!\! \neg F_{\permute(b)}  ),
&&\!\!\!\!  \text{since } c \in C \iff \permute(c) \in C^\permute \\
& = \bigvee_{C^\permute \in \mathfrak{C}_{U,k}^\permute}  (  \bigwedge_{\permute(a) \in C^\permute} \!\!\!\!\! F_{\permute(a)} \;. \!\!\! \bigwedge_{\permute(b) \not\in C^\permute} \!\!\!\!\! \neg F_{\permute(b)}  ),
&&\!\!\!\!  \text{since } C \in \mathfrak{C}_{U,k} \iff C^\permute \in \mathfrak{C}_{U,k}^\permute \\
& = \bigvee_{C' \in \mathfrak{C}_{U,k}^\permute}  (  \bigwedge_{a' \in C'} \!\! F_{a'} \;. \!\! \bigwedge_{b' \not\in C'} \!\!\! \neg F_{b'}  ),
&&\!\!\!\!  \text{renaming $C^\permute$ to $C'$ and $\permute(c)$ to $c'$} \\
& = \bigvee_{C' \in \mathfrak{C}_{U,k}}  (  \bigwedge_{a' \in C'} \!\!\! F_{a'} \;. \! \bigwedge_{b' \not\in C'} \!\!\! \neg F_{b'}  ),
&&\!\!\!\! \text{since $\mathfrak{C}_{U,k}^\permute = \mathfrak{C}_{U,k}$, (Lemma \ref{lem:bigC_unique})} \\
& = \textsc{Exact}(k, U, F), 
&& \!\!\!\! \text{by definition.} 
\end{align*}
\end{proof}
\begin{definition}
For each $C \subseteq U$, the proposition $Z_{(C,F)}$ is defined as:
\begin{equation} \label{eq:def_z}
Z_{(C, F)}:=
\bigwedge_{a \in C} \! F_a \;.
\bigwedge_{b \not\in C} \! \neg F_b
\end{equation} 
\end{definition}
which allows to write definition (\ref{eq:exact_def}) as: $\textsc{Exact}(k, U, F)  := \bigvee_{C \in \mathfrak{C}_{U,k}} Z_{(C, F)}$
\begin{lemma}\label{lem:pairwise}
For arbitrary propositions $A$ and $B$ and $1 \leq k \leq \usize$:
\begin{equation} \label{eq:exact_to_sum} 
pr(\textsc{Exact}(k, U, F) . A | B ) = \sum_{C \in \mathfrak{C}_{U,k}}{pr(Z_{(C, F)} . A | B)}
\end{equation}
\end{lemma}
\begin{proof} $\forall C' \neq C'' \in \mathfrak{C}_{U,k}$:
\begin{align*}
&\exists b \in U \quad b \in C' \text{ and } b \not \in C'',                                             && \!\!\!\! \text {$C'$\,\&\,$C''$\,being distinct with same size}\\
&\Longrightarrow
Z_{(C',F)} \models F_b \text{ and } Z_{(C'' , F)} \models \neg F_b,                     && \!\!\!\! \text{by definition (\ref{eq:def_z})}\\
&\Longrightarrow
Z_{(C' ,F)} . Z_{(C'' , F)} \models F_b \,. \neg F_b \equiv \bot 
\end{align*}
Therefore, the sequence $\{ Z_{(C, F)} \}_{C \in \mathfrak{C}_{U,k}}$ consists of mutually disjoint events. Hence, by the third Kolmogorov probability axiom, for all propositions $A$ and $B$: 
\begin{equation} \label{eq:exact_break}
pr(\textsc{Exact}(k, U, F) . A | B) = 
pr \big(\!\! \bigvee_{C \in \mathfrak{C}_{U,k}} \!\!\!\! (Z_{(C , F)} . A) | B \big) = 
\!\! \sum_{C \in \mathfrak{C}_{U,k}} \!\!\! pr(Z_{(C, F)} . A | B)
\end{equation}
\end{proof}
\begin{lemma}\label{lem:sum_to_one}
Let $1 \leq k \leq \usize$. For all $C \in \mathfrak{C}_{U,k}$ and for all permutations $\permute$: 
\begin{equation*}
Z_{(C, F)}^{\permute} = Z_{(C^{\permute}, F)}
\end{equation*} 
\end{lemma}
\begin{proof} 
\begin{align*}
Z_{(C, F)}^{\permute} &= 
\bigwedge_{a \in C} \!\! F_{\permute(a)} \;.
\bigwedge_{b \not\in C} \!\! \neg F_{\permute(b)} 
&& \text{by relation (\ref{eq:def_z})} \\
& = \bigwedge_{\permute(a) \in C^{\permute}} \!\!\!\!\! F_{\permute(a)} \; . \bigwedge_{\permute(b) \not\in C^{\permute}} \!\!\!\!\!\! \neg F_{\permute(b)} 
&& \text{by def. } c \in C   \iff   \permute(c) \in C^{\permute}\\
& = \bigwedge_{a' \in C^{\pi}} \!\!\! F_{a'} \; . \bigwedge_{b' \not\in C^{\pi}} \!\!\!\! \neg F_{b'} 
&& \text{renaming $\permute (a)$ to $a'$ and $\permute (b)$ to $b'$}\\
& = Z_{(C^{\permute}, F)}         		
&&\text{by relation (\ref{eq:def_z})}.
\end{align*}
\end{proof}
\begin{proof}[{\bf Proof of Theorem \ref{theo:exact2list}}] For the sake of conciseness, here we only prove relation (\ref{eq:exact2list_fa_ga}) i.e.:
\begin{equation*}
\forall a \in U, 1 \leq k \leq \usize \quad pr \big(H | \textsc{Exact}(k, U, F) \, . F_a.  G_a \big) = pr \big(H | F_{1 : k} .\neg  F_{{k +1} : {\usize}} . G_{k} \big)
\end{equation*}
The remaining relations (\ref{eq:exact2list}, \ref{eq:exact2list_fa_notga}  \& \ref{eq:exact2list_notfa_notga}) can be proved using the same method and by small (and obvious) appropriate modifications.

For the sake of simplicity, we first swap the names of the objects $a$ and $1$ as follows: 
Let $\permute' := \{ a / 1 ; 1 / a\}$. By the exchangeability assumption: 
\begin{align} \label{eq:star}
&pr \big(H | \, \textsc{Exact}(k, U, F) \, . F_a.G_a \big) = pr \big(H^{\permute'} | \, \textsc{Exact}^{\permute'}(k, U, F) \, . F_{\permute' (a)}.G_{\permute' (a)} \big) \notag\\
&=pr \big( H | \textsc{Exact}(k, U, F) \, . F_{1}.G_{1} \big), 
&&\!\!\!\!\!\!\!\!\!\!\!\!\!\!\!\!\!\!\!\!\!\!\!\!\!\!\!\!\!\!\!\!\!\!\!\!\!\!\!\!\!\!\!\!\!\!\!\!\!\!\!\!\!\!\!\!\! \text{by Lemmas \ref{lem:h_permute} and \ref{lem:exact_permute}} \notag\\
&= \frac{pr(H) \cdot pr\big( \textsc{Exact}(k, U, F) . \, F_{1}.G_{1} | H \big)}{pr \big( \textsc{Exact}(k, U, F) \, . F_{1}.G_{1} \big)}, 
&&\!\!\!\!\!\!\!\!\!\!\!\!\!\!\!\!\!\!\!\!\!\!\!\!\!\!\!\!\!\!\!\!\!\!\!\!\!\!\!\!\!\!\!\!\!\!\!\!\!\!\!\!\!\!\!\!\! \text{by Bayes rule.}
\end{align}
Now:
\begin{align} \label{eq:exact_expand1}
&pr \big( \textsc{Exact}(k, U, F) \, . F_{1}.G_{1} | H \big) =\!\! \sum_{C \in \mathfrak{C}_{U,k}}\!\!{pr( Z_{(C, F)} .  F_{1}.G_{1} | H)}, \quad \text{by Lemma \ref{lem:pairwise}}\notag\\
&=\!\! \sum_{C \in \mathfrak{C}_{U,k} \text{ s.t.\ } 1 \in C}\!\!\!\!\!\!\!\!\!\!{pr(Z_{(C, F)} .  F_{1}.G_{1} | H)} 
\;\;\;\;\;\;+ \!\! \sum_{C \in \mathfrak{C}_{U,k} \text{ s.t.\ } 1 \not \in C }\!\!\!\!\!\!\!\!\!\!{pr(Z_{(C, F)} .  F_{1}.G_{1} | H)}\notag\\
&=\!\! \sum_{C \in \mathfrak{C}_{U,k} \text{ s.t.\ } 1 \in C}\!\!\!\!\!\!\!\!\!\!{pr(Z_{(C, F)} .  F_{1}.G_{1} | H)} 
\end{align}
The second summation is eliminated because in the case of any $C \in \mathfrak{C}_{U,k}$ such that $1 \not \in C$, $Z_{(C, F)} \vDash \neg F_{1}$ (see definition (\ref{eq:def_z})), therefore in this case, $pr(Z_{(C, F)} .  F_{1}.G_{1} | H) = pr(Z_{(C, F)} .  \neg F_{1} . F_{1}.G_{1} | H)  = pr(\bot | H ) = 0$.

Due to the exchangeability assumption, the members of the first summation are all equal. The reasoning is as follows:\\
In the case of any $C \in \mathfrak{C}_{U,k}$ such that $1 \in C$, there exists some permutation that map each of the members of $C$ to the set $\{1, 2, \ldots k \}$ with $1$ as a fixed point i.e.: 
\begin{equation} \label{eq:special_map}
\forall C \in \mathfrak{C}_{U,k} \text{ s.t.} \; 1 \in C: \quad 
\exists \permute \quad {C}^{\permute} = \{1, 2, \ldots, {k} \}, \; \permute (1) = 1  
\end{equation}
In fact, for any member of $\mathfrak{C}_{U,k}$, exactly $(k-1)!$ permutations with such properties exist. Using such permutations, $\forall C \in \mathfrak{C}_{U,k} \text{ s.t. } 1 \in C$:
\begin{align} \label{eq:expand2}
\exists \permute \quad &pr(Z_{(C , F)} .  F_{1}.G_{1} | H) = pr(Z_{(C, F)}^{\permute} .  F_{\permute (1)}.G_{\permute (1)} | H^{\permute}), 
&&\!\!\! \text{by exchangeability}\notag\\
&\quad = pr(Z_{( \{1, \ldots, k\} , F)} .  F_{\permute (1)}.G_{\permute (1)} | H^{\permute}), 
&&\!\!\! \text{using (\ref{eq:special_map}) in Lemma \ref{lem:sum_to_one}}\notag\\
&\quad = pr(Z_{(\{1, \ldots, k\} , F)} .  F_{1}.G_{1} | H^{\permute}), 
&&\!\!\! \text{$1$ being a fix point}\notag\\
&\quad = pr(Z_{(\{1, \ldots, k\} , F)} .  F_{1}.G_{1} | H),                    
&&\!\!\! \text{by Lemma \ref{lem:h_permute}}\notag\\
&\quad = pr(F_{1 : k} . \neg F_{{k+1} : {\usize}} .  G_{1} | H),                    
&&\!\!\! \text{expanding $Z_{(.,.)}$ by def. (\ref{eq:def_z})}
\end{align}
Thus, combining (\ref{eq:exact_expand1}) and (\ref{eq:expand2}):
\begin{align} \label{eq:sun}
&pr \big (\textsc{Exact}(k, U, F) . F_{1}.G_{1} | H \big ) = 
\sum_{C \in \mathfrak{C}_{U,k} \text{ s.t.\  } 1 \in C}\!\!\!\!\!\!\!\!\!\! {pr(Z_{(C, F)} .  F_{1}.G_{1} | H)} \notag\\
&\quad= {\usize-1\choose{k-1}} pr(F_{1 : k} . \neg F_{{k+1} : {\usize}} .  G_{1} | H) 
\end{align}
By a similar justification:
\begin{equation} \label{eq:moon}
pr \big (\textsc{Exact}(k, U, F) . F_{1}.G_{1} \big ) = {\usize-1\choose{k-1}} pr(F_{1 : k} . \neg F_{{k+1} : {\usize}} .  G_{1})
\end{equation}
Combining equations (\ref{eq:sun}), (\ref{eq:moon}) and (\ref{eq:star}):
\begin{align*}
&pr \big (\textsc{Exact}(k, U, F) \, . F_{1}.G_{1} | H \big ) = \frac{pr(H) \cdot pr \big( \textsc{Exact}(k, U, F) \, . F_{1}.G_{1} | H \big)} {pr \big( \textsc{Exact}(k, U, F) \, . F_{1}.G_{1} \big)} \\
&=\frac{{\usize-1\choose{k-1}} \, pr(H) \, pr(F_{1 : k} \neg F_{{k+1} : {\usize}} .  G_{1} | H) }
{{\usize-1\choose{k-1}} pr(F_{1 : k} \neg F_{{k+1} : {\usize}} . G_{1})} 
=pr(H| F_{1 : k} \neg F_{{k+1} : {\usize}} .  G_{1})
\end{align*}
\end{proof}
\begin{small}

\appendix

\section{List of Notation}\label{app:Notation}

\begin{tabbing}
  \hspace{0.23\textwidth} \= \hspace{0.63\textwidth} \= \kill 
 {\bf Symbol }      \> {\bf Explanation}                                                    \\[0.5ex]
  $U = \{1, 2, \ldots ,\upsilon \}$				\>universe of arbitrary size $\usize$\\[0.5ex]
  $1, 2, \ldots, \usize$		\>objects of universe $U$ (short form)\\[0.5ex]
  $\upsilon, \upsilon', \upsilon'' \in \mathbb{N} $		\>symbols used to denote the size of universe\\[0.5ex]
  $U_\upsilon = \{1, 2, \ldots ,\upsilon \}$			\>universe of size $\upsilon$\\[0.5ex]
  $\alpha \leq \beta$			\>lower and higher bounds for the size of the universe\\[0.5ex]
  $a, b, c$ and $b_1, b_2, \ldots$			\>typical objects (or individuals) (not necessarily consecutive)\\[0.5ex] 
  $\psi$			\>typical 1-place predicate\\[0.5ex]
  $\psi_b$			\>a proposition assigning predicate $\psi$ to object $b$\\[0.5ex] 
  $\psi_{b_i : b_j}$ 		\>$\psi_{b_i} . \psi_{b_{i+1}} \ldots \psi_{b_j}$\\[0.5ex]  
  $F, G$			\>atomic 1-place predicates\\[0.5ex]
  $F_b, G_b$			\>propositions assigning $F$ and $G$ to object $b$, respectively\\[0.5ex]
  $\Qff_b, \Qft_b, \Qtf_b, \Qtt_b$	\> complete descriptions (of object $b$)\\[0.5ex]
  $\impl_b$			\>$F_b \rightarrow G_b \equiv \neg F_b \vee G_b \equiv \neg \Qtf_b$ \\[0.5ex]
  $H$				\>General hypothesis: $\forall b \in U \quad \impl_b$\\[0.5ex]
  $E$				\>evidence: $F_a . G_a$\\[0.5ex]
  $B, D$				\>typical (objective) background knowledge\\[0.5ex]
  $\rho$			\>a typical proposition\\[0.5ex]
  $\inds{\rho}$		\>set of all individuals described by $\rho$\\[0.5ex]
  $\preds{\rho}$		\>set of all (simple) predicates involved in $\rho$\\[0.5ex]
  $\top$			\>tautologous proposition\\[0.5ex]
  $\Delta$			\>set of all propositions in form of relation (\ref{eq:artless}) (on universe $U$)\\[0.5ex]
  $\Delta_\upsilon$		\>set of all propositions in form of relation (\ref{eq:artless}) (on universe $U_\upsilon$)\\[0.5ex]
  $\delta$			\>set of all complete descriptions that do not falsify $H$\\[0.5ex]
  CDV			\>complete description vector\\[0.5ex] 
  $\Omega$			\>set of all CDVs (w.r.t\ universe $U$)\\[0.5ex]
  $\Omega_\upsilon$	\> set of all CDVs (w.r.t.\ universe $U_\upsilon$)\\[0.5ex]
  $\bigss$			\>union of $\Omega_\alpha$ to $\Omega_\beta$\\[0.5ex]
  $\omega^{\Omega_\upsilon}_{\rho} \subseteq \Omega_\upsilon$ 	\>an event that corresponds proposition $\rho$ (w.r.t.\ sample space $\Omega_\upsilon$)\\[0.5ex]
  $\omega_{\rho} \subseteq \bigss$		\>an event that corresponds proposition $\rho$ (w.r.t.\ sample space $\bigss$)\\[0.5ex]
  $\measure$				\>probability (over sample space $\Omega$)\\[0.5ex]  
  $\measure_\upsilon$		\>probability (over sample space $\Omega_\upsilon$)\\[0.5ex]  
  $\bigmeasure$			\>probability (over sample space $\bigss$)\\[0.5ex]  
  $\permute$, $\permute'$, $\permute''$				\>typical permutations (i.e.\ bijections) in $U$\\[0.5ex]
  $\pi(b)$			\>an object that $b \in U$ is mapped to by bijection $\pi$\\[0.5ex]
  $\rho^\permute$		\>a proposition obtained from $\rho$ by replacing any $b \in \inds_\rho$ with $\permute(b)$\\[0.5ex]
  $C, C', C'' \subseteq U$		\>typical subsets of $U$\\[0.5ex]
  $C^\permute$		\>$\{ \permute (b)  :  b \in C \}$\\[0.5ex]
  $\mathfrak{C}_{U,k}$	\>set of all subsets of $U$ with cardinality $k$\\[0.5ex]
  $\mathfrak{C}_{U,k}^\permute$	\>$\{ C^\permute  :  C \in \mathfrak{C}_{U,k} \}$\\[0.5ex]
  $Z_{(C, F)}$ 		\>$\bigwedge_{a \in C} \! F_a \;. \bigwedge_{b \not\in C} \! \neg F_b$\\[0.5ex]
  $\textsc{Exact}(k, U, F)$\>a proposition representing: ``exactly $k$ members of $U$ are $F$" \\[0.5ex]
\end{tabbing}

\end{small}

\end{document}